\documentclass[11pt]{article}
\usepackage{amsmath}
\usepackage{amsthm}
\usepackage{amssymb}
\usepackage{graphicx}
\usepackage{mathtools}
\usepackage[margin=2cm]{geometry}
\usepackage{setspace}
\usepackage{epstopdf}
\usepackage{subfig}
\usepackage{empheq}
\usepackage{algorithm}
\usepackage{algorithmicx}
\usepackage{algpseudocode}
\usepackage{multirow}
\usepackage{url}
\usepackage[title]{appendix}
\usepackage{enumitem}
\usepackage[svgnames]{xcolor}
\usepackage{mathtools}
\usepackage[mathlines]{lineno}

\makeatletter

\theoremstyle{plain}
\newtheorem{thm}{\protect\theoremname}
  \theoremstyle{definition}

  \theoremstyle{plain}
  \newtheorem{lem}[thm]{\protect\lemmaname}
  \newtheorem{prop}[thm]{\protect\propname}

  \providecommand{\propname}{Proposition}
  \providecommand{\condname}{Condition}
  \providecommand{\probname}{Problem}

\makeatother

  \providecommand{\definitionname}{Definition}
  \providecommand{\lemmaname}{Lemma}
  \providecommand{\corollaryname}{Corollary}
\providecommand{\theoremname}{Theorem}

\usepackage{endnotes}
\usepackage{color}
\usepackage{xparse}

\newtheorem{remark}{Remark}

\title{Doubly-Stochastic Normalization of the Gaussian Kernel \\ is Robust to Heteroskedastic Noise}

\author{ Boris Landa${^{1,*}}$~~~~Ronald R.Coifman${^{1}}$~~~~Yuval Kluger${^{1,2,3}}$\\
\small{${^1}$Program in Applied Mathematics, Yale University}\\
\small{${^2}$Interdepartmental Program in Computational Biology and Bioinformatics, Yale University}\\
\small{${^3}$Department of Pathology, Yale University School of Medicine}\\
\small{${^*}$Corresponding author. Email: boris.landa@yale.edu}
}

\begin{document}

\maketitle

\begin{abstract}
{A fundamental step in many data-analysis techniques is the construction of an affinity matrix describing similarities between data points. When the data points reside in Euclidean space, a widespread approach is to from an affinity matrix by the Gaussian kernel with pairwise distances, and to follow with a certain normalization (e.g. the row-stochastic normalization or its symmetric variant). We demonstrate that the doubly-stochastic normalization of the Gaussian kernel with zero main diagonal (i.e., no self loops) is robust to heteroskedastic noise. That is, the doubly-stochastic normalization is advantageous in that it automatically accounts for observations with different noise variances. Specifically, we prove that in a suitable high-dimensional setting where heteroskedastic noise does not concentrate too much in any particular direction in space, the resulting (doubly-stochastic) noisy affinity matrix converges to its clean counterpart with rate $m^{-1/2}$, where $m$ is the ambient dimension. We demonstrate this result numerically, and show that in contrast, the popular row-stochastic and symmetric normalizations behave unfavorably under heteroskedastic noise. Furthermore, we provide examples of simulated and experimental single-cell RNA sequence data with intrinsic heteroskedasticity, where the advantage of the doubly-stochastic normalization for exploratory analysis is evident.}
\end{abstract}

\section{Introduction}
\subsection{Affinity matrix constructions}
Given a dataset of points in Euclidean space, a useful approach for encoding the intrinsic geometry of the data is by a weighted graph, where the vertices represent data points, and the edge-weights describe similarities between them. Such a graph can be described by an affinity (or adjacency/similarity) matrix, namely a nonnegative matrix whose $(i,j)$'th entry holds the edge-weight between vertices $i$ and $j$. To measure the similarity between pairs of data points, one can employ the Gaussian kernel with pairwise (Euclidean) distance. In particular, given data points $\mathbf{x}_1,\ldots,\mathbf{x}_n \in \mathbb{R}^m$, we consider the matrix $K\in\mathbb{R}^{n \times n}$ given by
\begin{equation}
    K_{i,j} = \begin{dcases}
    \operatorname{exp}(-\Vert \mathbf{x}_i - \mathbf{x}_j\Vert^2/\varepsilon), & i\neq j, \\
    0, & i=j, 
    \end{dcases} \label{eq:K def}
\end{equation}
for $i,j=1,\ldots,n$, where $\varepsilon$ is the \textit{kernel width} parameter. For many applications it is a common practice to normalize $K$, so to equip the resulting affinity matrix with a useful interpretation and favourable properties. Two such normalizations, which are closely related to each other, are the \textit{row-stochastic} and the \textit{symmetric} normalizations:
\begin{flalign}
&    \text{(\textbf{Row-stochastic normalization})}   &
    W^{(\text{r})} &\overset{\text{def}}{=} \operatorname{diag}(\textbf{r})  K, \qquad\qquad r_i = \frac{1}{\sum_{j=1}^n K_{i,j}},  &   \label{eq:W_r def}\\
&    \text{(\textbf{Symmetric normalization})} &
   W^{(s)} & \overset{\text{def}}{=} \sqrt{\operatorname{diag}(\mathbf{r})} K \sqrt{\operatorname{diag}(\mathbf{r})},  &   \label{eq:W_s def}
\end{flalign}
where $\mathbf{r}= [r_1,\ldots,r_n]$, and $\operatorname{diag}(\mathbf{r})$ is a diagonal matrix with $\mathbf{r}$ on its main diagonal.

Notably, the matrix $W^{(\text{r})}$ is row-stochastic, i.e., the sum of every row of $W^{(\text{r})}$ is $1$, which allows for a useful interpretation of $W^{(r)}$ as a transition-probability matrix (in the sense of a Markov chain). 
An important characteristic of the row-stochastic affinity matrix $W^{(r)}$ is its relation to the heat kernel and the Laplace-Beltrami operator on a manifold~\cite{belkin2003laplacian,coifman2006diffusionMaps,hein2005graphs,singer2006graph,trillos2019error}. Specifically, under the ``manifold assumption'' -- where the points $\mathbf{x}_1,\ldots,\mathbf{x}_n$ are uniformly sampled from a smooth low-dimensional Riemannian manifold embedded in the Euclidean space -- $W^{(r)}$ approximates the heat kernel on the manifold, and the matrix $L^{(r)}=I-W^{(r)}$ (known as the \textit{random-walk graph Laplacian}) approximates the Laplace-Beltrami operator. This property of the row-stochastic normalization establishes the relation between $W^{(r)}$ and the intrinsic local geometry of the data, thereby justifying the use of $W^{(r)}$ as an affinity matrix.

The affinity matrix $W^{(s)}$ (obtained by the symmetric normalization) is closely-related to $W^{(r)}$, and in particular, since $W^{(s)} = [\operatorname{diag}(\textbf{r})]^{-1/2} W^{(r)} [\operatorname{diag}(\textbf{r})]^{1/2}$, $W^{(s)}$ shares the spectrum of $W^{(r)}$, and their eigenvectors are related through the vector $\mathbf{r}$. Even though $W^{(s)}$ is not a proper transition-probability matrix, it enjoys symmetry, which is advantageous in various applications. 

We also mention that the row stochastic and symmetric normalizations can be used in conjunction with a kernel with variable width, i.e., when a different value of $\varepsilon$ is taken for each row or column of $K$ (see for instance~\cite{berry2016variable} and references therein). We further discuss one such variant in the example in Section~\ref{subsec:SCRNAseq example}.

The matrices $W^{(r)}$ and $W^{(s)}$ (or equivalently, their corresponding graph Laplacians $I-W^{(r)}$ and $I-W^{(s)}$) are used extensively in data processing and machine learning, notably in non-linear dimensionality reduction (or manifold learning)~\cite{belkin2003laplacian,coifman2006diffusionMaps,nadler2006diffusion,maaten2008visualizing}, community detection and spectral-clustering ~\cite{shi2000normalized,ng2002spectral,sarkar2015role,von2007tutorial,fortunato2010community,shaham2018spectralnet,kluger2003spectral}, image denoising~\cite{buades2005non,pang2017graph,meyer2014perturbation,landa2018steerable,singer2009diffusion}, and in signal processing and supervised-learning over graph domains~\cite{shuman2013emerging,coifman2006diffusionWavelets,hammond2011wavelets,defferrard2016convolutional,bronstein2017geometric}.

\subsection{The doubly-stochastic normalization}
In this work, we focus on the \textit{doubly-stochastic} normalization of $K$:
\begin{flalign}
&\text{(\textbf{Doubly-stochastic normalization})} &   W^{(\text{d})} &\overset{\text{def}}{=} \operatorname{diag}(\textbf{d})  K \operatorname{diag}(\textbf{d}), & \label{eq:W_d def}
\end{flalign}
where $\mathbf{d}=[d_1,\ldots,d_n] > 0$ is a vector chosen such that $W^{(d)}$ is doubly-stochastic, i.e., such that the sum of every row and every column of $W^{(\text{d})}$ is $1$. 
The problem of finding $\mathbf{d}$ such that $W^{(d)}$ has prescribed row and column sums is known as a \textit{matrix scaling} problem, and the entries of $\mathbf{d}$ are often referred to as \textit{scaling factors}. Matrix scaling problems have a rich history, with a long list of applications and generalizations~\cite{bapat1997nonnegative,idel2016review}. Since the scaling factors are defined implicitly, their existence and uniqueness are not obvious, and depend on the zero-pattern of the matrix to be scaled. For the particular zero-pattern of $K$, existence and uniqueness are established by the following lemma.
\begin{lem}[Existence and uniqueness] \label{lem:existence and uniquness}
Suppose that $A \in \mathbb{R}^{n\times n}$, $n>2$, is symmetric with zero main diagonal and strictly positive off-diagonal entries. Then, there exist scaling factors $d_1,\ldots,d_n > 0$ such that $\sum_{j=1}^n d_i A_{i,j} d_j = 1$ for all $i=1,\ldots,n$, and moreover, $\{d_i\}_{i=1}^n$ are unique. 
\end{lem}
The proof can be found in Appendix~\ref{appendix:existence and uniqueness proof}, and is based on the simple zero-pattern of $A$ and on a lemma by Knight~\cite{knight2008sinkhorn}. On the computational side, the scaling factors $\mathbf{d}$ can be obtained by the classical Sinkhorn-Knopp algorithm~\cite{sinkhorn1967concerning} (known also as the RAS algorithm), or by more recent techniques based on optimization (see~\cite{allen2017much} and references therein). We detail a lean variant of the Sinkhorn-Knopp algorithm adapted to symmetric matrices (see~\cite{knight2008sinkhorn}) in Algorithm~\ref{alg:SK sym} below, and briefly discuss its convergence and computational complexity in Remark~\ref{remark:compuational complexity}. 
\begin{algorithm}
\caption{Sinkhorn-Knopp algorithm for symmetric matrices~\cite{knight2008sinkhorn}}\label{alg:SK sym}
\begin{algorithmic}[1]
\Statex{\textbf{Input:} Symmetric nonnegative $n\times n$ matrix $K$, tolerance $\delta>0$}.
\State Initialize: $d_{i}^{(0)} = (\sum_{j=1}^n K_{i,j})^{-1}$, $d_{i}^{(1)} = (\sum_{j=1}^n K_{i,j} d_{j}^{(0)})^{-1}$, $d_{i}^{(2)} = (\sum_{j=1}^n K_{i,j} d_{j}^{(1)})^{-1}$, $\tau = 2$.
\State While $\max_{1\leq i \leq n}\vert d_{i}^{(\tau-2)} / d_{i}^{(\tau)} - 1 \vert > \delta$, do:
\begin{itemize}
    \item $d_{i}^{{(\tau+1)}} = ({\sum_{j=1}^n K_{i,j} d_{j}^{(\tau)}})^{-1}$, for $i=1,\ldots,n$.
    \item Update $\tau \leftarrow \tau + 1$.
\end{itemize}
\State Return $d_i = \sqrt{d_i^{(\tau)} d_i^{(\tau-1)}}$, for $i=1,\ldots,n$.
\end{algorithmic}
\end{algorithm}

By definition, $W^{(d)}$ is a symmetric transition-probability matrix. Hence, it naturally combines the two favorable properties that $W^{(r)}$ and $W^{(s)}$ hold separately. It is worthwhile to point-out that $W^{(d)}$ is in fact the closest symmetric and row-stochastic matrix to $K$ in KL-divergence~\cite{brown1993order,zass2007doubly}, and interestingly, it can also be obtained by iteratively re-applying the symmetric normalization~\eqref{eq:W_s def} indefinitely (see~\cite{zass2005unifying}). 
Another appealing interpretation of the doubly-stochastic normalization is through the lens of optimal transport with entropy regularization~\cite{cuturi2013sinkhorn}, summarized by the following proposition.
\begin{prop} [Optimal transport interpretation]\label{prop:optimal trnasport interpretation}
 $W^{(d)}$ from~\eqref{eq:W_d def} is the unique solution to
\begin{align}
\begin{aligned}
    &\underset{W\in \mathbb{R}^{n\times n}_+}{\text{Minimize}} \quad \sum_{i,j=1}^n \Vert \mathbf{x}_i -\mathbf{x}_j \Vert^2 W_{i,j} + \varepsilon H(W), \\
    &\text{Subject to:} \quad W \mathbf{1} = \mathbf{1}, \quad W^T \mathbf{1} = \mathbf{1}, \quad W_{i,i} = 0, \;\; i=1\ldots,n,
    \end{aligned} \label{eq:optimal trnasport optim}
\end{align}
where $\mathbf{1}$ is a column vector of $n$ ones, and $H(W) = \sum_{i,j=1}^n W_{i,j} \log W_{i,j}$ is the negative entropy.
\end{prop} 
The proof of Proposition~\ref{prop:optimal trnasport interpretation} follows very closely with the proof of Lemma 2 in~\cite{cuturi2013sinkhorn}, with the additional use of Lemma~\ref{lem:existence and uniquness} (to account for the constraint $W_{i,i} = 0$), and is omitted for the sake of brevity.
In the optimal transport interpretation of the problem~\eqref{eq:optimal trnasport optim}, each point $\mathbf{x}_i$ holds a unit mass that should be distributed between all the other points $\mathbf{x}_j \neq \mathbf{x}_i$, while minimizing the transportation cost between the points (measured by the pair-wise distances $\Vert \mathbf{x}_i -\mathbf{x}_j \Vert^2$). The outcome of this process is constrained so that each point ends up holding a unit mass. In this context, the matrix $W$ describes the distribution of the masses from all points to all other points, and is therefore required to be doubly-stochastic. The negative entropy regularization term $\varepsilon H(W)$ controls the ``fairness'' of the mass allocation, such that each mass is distributed more evenly between the points for large values of $\varepsilon$. 

The optimization problem~\eqref{eq:optimal trnasport optim} can also be interpreted as an optimal graph construction. In this context, the term $\sum_{i,j=1}^n \Vert \mathbf{x}_i -\mathbf{x}_j \Vert^2 W_{i,j}$ can be considered as accounting for the regularity of the data (as a multivariate signal) with respect to the weighted graph represented by $W$, while the negative entropy term $\varepsilon H(W)$ controls the approximate sparseness of $W$. Since the solution to~\eqref{eq:optimal trnasport optim} is a symmetric matrix, $W^{(d)}$ can be thought of as describing the undirected weighted graph that optimizes the ``smoothness'' of the dataset, under the constraints of prescribed entropy (or approximate sparseness), no self-loops, and stochasticity (i.e., so that $W^{(d)}$ is a transition-probability matrix).

In the context of manifold learning, the relation between the doubly-stochastic normalization and the heat kernel (or the Laplace-Beltrami operator) on a Riemannian manifold has been recently established in~\cite{marshall2019manifold}. That is, under the manifold assumption (and under certain conditions) $W^{(d)}$ is expected to approximate the heat kernel on the manifold, and therefore to encode the local geometry of the data much like $W^{(r)}$. The doubly-stochastic normalization was also demonstrated to be useful for spectral clustering in~\cite{beauchemin2015affinity}, where it was shown to achieve the best clustering performance on several datasets. Last, we note that several other constructions of doubly-stochastic affinity matrices have appeared in the literature~\cite{wang2012improving,zass2007doubly}, typically involving a notion of closeness to $K$ other than KL-divergence (e.g. Frobenius norm).

\begin{remark}[Computational complexity of Algorithm~\ref{alg:SK sym}] \label{remark:compuational complexity} 
   It is evident that the computational complexity of each iteration in Algorithm~\ref{alg:SK sym} is dominated by the matrix-vector multiplication $K \mathbf{d}^{(\tau)}$, and is therefore $\mathcal{O}(n^2)$. As for the number of iterations required, in~\cite{knight2008sinkhorn} it was shown that if the matrix to be scaled is \textit{fully indecomposable}, then the scaling factors in the Sinkhorn-knopp algorithm admit a linear convergence whose rate is equal to the squared subdominant eigenvalue of the resulting doubly-stochastic matrix (see Theorem 4.4 in~\cite{knight2008sinkhorn}). In the proof of Lemma~\ref{lem:existence and uniquness} in Appendix~\ref{appendix:existence and uniqueness proof} we show that $K$ from~\eqref{eq:K def} is indeed fully indecomposable, hence the number of iterations in Algorithm~\ref{alg:SK sym} is expected to be $\mathcal{O}(1/\log (|\lambda_2\{W^{(d)}\}|^{-1}))$, where $\lambda_2\{W^{(d)}\}$ is the subdominant eigenvalue of $W^{(d)}$.
\end{remark}

\subsection{Robustness to noise}
When considering real-world datasets, it is desirable to construct affinity matrices that are robust to noise. 
Specifically, suppose that we do not have access to the points $\mathbf{x}_1,\ldots,\mathbf{x}_n$ (which are non-random in our setting), but rather to their noisy observations $\widetilde{\mathbf{x}}_1,\ldots,\widetilde{\mathbf{x}}_n$, given by
\begin{equation}
    \widetilde{\mathbf{x}}_i = \mathbf{x}_i + \eta_{i},  \label{eq:x_tilde def}
\end{equation}
where $\eta_{1},\ldots,\eta_n \in \mathbb{R}^m$ are pairwise independent noise vectors satisfying
\begin{equation}
    \mathbb{E}[\eta_i] = \mathbf{0}, \qquad \mathbb{E}[\eta_i \eta_i^T] = \Sigma_i, \label{eq:noise def}
\end{equation}
for all $i=1,\ldots,n$, where $\mathbf{0}$ is the zero column vector in $\mathbb{R}^m$, and $\Sigma_i$ is the covariance matrix of $\eta_i$. We then define $\widetilde{W}^{(r)}$, $\widetilde{W}^{(s)}$, $\widetilde{W}^{(d)}$, $\widetilde{K}$, and $\{\widetilde{d}_i\}$ analogously to ${W}^{(r)}$, ${W}^{(s)}$, ${W}^{(d)}$, $K$ , and $\{{d}_i\}$, respectively, when replacing $\{\mathbf{x}_i\}$ in~\eqref{eq:K def} with $\{\widetilde{\mathbf{x}}_i\}$.
For the noise model described above, we say that the noise is \textit{homoskedastic} if $\Sigma_1 = \Sigma_2 = \ldots = \Sigma_n$, and \textit{heteroskedastic} otherwise.

The influence of homoskedastic noise on kernel matrices (such as $K$) was investigated in~\cite{el2010information}, and the results therein imply that $\widetilde{W}^{(r)}$ and $\widetilde{W}^{(s)}$ are robust to high-dimensional homoskedastic noise. Specifically, in the high-dimensional setting considered in~\cite{el2010information}, $\widetilde{K}$ converges to a biased version $K$ where all the off-diagonal entries of $\widetilde{K}$ admit the same multiplicative bias. Such bias can therefore be corrected by applying either the row-stochastic or the symmetric normalizations (see~\cite{el2016graph}). However, this is not the case in the more general setting of heteroskedastic noise. 

Heteroskedastic noise is a natural assumption for many real-world applications. For example, heteroskedastic noise arises in certain biological, photon-imaging, and Magnetic Resonance Imaging (MRI) applications~\cite{cao2017multi,salmon2014poisson,hafemeister2019normalization,foi2011noise}, where observations are modeled as samples from random variables whose variances depend on their means, such as in binomial, negative-binomial, multinomial, Poisson, or Rice distributions. In natural image processing, heteroskedastic noise occurs due to the spatial clipping of values in an image~\cite{foi2009clipped}. Additionally, heteroskedastic noise is encountered when the experimental setup varies during the data collection process, such as in spectrophotometry and atmospheric data acquisition~\cite{cochran1977statistically,tamuz2005correcting}. Generally, many modern datasets are inherently heteroskedastic as they are formed by aggregating observations collected at different times and from different sources. Last, we mention that heteroskedastic noise can be considered as a natural relaxation to the popular manifold assumption. In particular, heteroskedastic noise arises whenever data points are sampled from the high-dimensional surroundings of a low-dimensional manifold embedded in the ambient space, where the size of the sampling neighborhood (in the ambient space) around the manifold is determined locally by the manifold itself. See Figure~\ref{fig:uniform noise varying radius 2d} and the corresponding example in Section~\ref{subsubsec:circle example 2}.

\subsection{Contributions}
Our main contribution is to establish the robustness of the doubly-stochastic normalization of the Gaussian kernel (with zero main diagonal) to high-dimensional heteroskedastic noise. 
In particular, we prove that in the high-dimensional setting where the number of points $n$ is fixed, the dimension $m$ is increasing, and the noise does not concentrate too much in specific direction in space, $\widetilde{W}^{(d)}$ converges to $W^{(d)}$ with rate $m^{-1/2}$. See Theorem~\ref{thm:Main theorem} in Section~\ref{sec:main result}. An intuitive justification of the robustness of the doubly-stochastic normalization to heteroskedastic noise, and also why zeroing-out the main diagonal of $K$ is important, can be found in Section~\ref{sec:main result}, equations~\eqref{eq:noisy distances bias}--\eqref{eq:K_tilde diagonal bias}.
The proof of Theorem~\ref{thm:Main theorem}, see Appendix~\ref{sec:proof of main theorem}, relies on a perturbation analysis of the doubly-stochastic normalization. 

We demonstrate the robustness of $\widetilde{W}^{(d)}$ to heteroskedastic noise in several examples (see Section~\ref{sec:numerical example}). In Section~\ref{subsubsec:circle example 1} we corroborate Theorem~\ref{thm:Main theorem} numerically, and exemplify that $W^{(r)}$ and $W^{(s)}$ suffer from inherent point-wise bias due to heteroskedastic noise (see Figures~\ref{fig:err vs m}--\ref{fig:W_tilde vs W heat map}). In Section~\ref{subsubsec:circle example 2} we demonstrate the robustness of the leading eigenvectors of $W^{(d)}$ to heteroskedastic noise whose characteristics depend locally on the manifold of the clean data (see Figures~\ref{fig:uniform noise varying radius 2d}--\ref{fig:embedding 2d}). In Section~\ref{subsec:SCRNAseq example} we apply the doubly-stochastic normalization for both simulated and experimental single-cell RNA sequence data with inherent heteroskedasticity, showcasing its ability to accurately recover the underlying structure of the data despite the noise (see Figures~\ref{fig:Noisy affinity matrices small epsilon},\ref{fig:t-SNE single-cell RNA seq},\ref{fig:scRNA-seq real data affinities},\ref{fig:scRNA-seq real data error in nearest neighbors}).

\section{Main result} \label{sec:main result}
We now place ourselves in the high-dimensional setting where the dimension $m$ is increasing while the number of points $n$ and the kernel parameter $\varepsilon$ are fixed. 
Formally, let $\mathbf{x}_i^{(m)}$, $\widetilde{\mathbf{x}}_i^{(m)}$, $\eta_i^{(m)}$, $\Sigma_{i}^{(m)}$, $K^{(m)}$, $\widetilde{K}^{(m)}$, ${W}^{(d),(m)}$, and $\widetilde{W}^{(d),(m)}$ be the same as $\mathbf{x}_i$, $\widetilde{\mathbf{x}}_i$, $\eta_i$, $\Sigma_{i}$, $K$, $\widetilde{K}$, ${W}^{(d)}$, and $\widetilde{W}^{(d)}$, respectively, and consider a sequence of each of the former quantities (with superscript $(m)$) in $m=M, M+1, \ldots, \infty$, where $M$ is a positive integer. Our main result is as follows, where $\mathcal{O}_p$ stands for \textit{order in probability}~\cite{mann1943stochastic} (or \textit{stochastic boundedness}). 
\begin{thm}[Convergence of $\widetilde{W}^{(d),(m)}$ to $W^{(d),(m)}$] \label{thm:Main theorem}
\sloppy Suppose that $\Vert \mathbf{x}_i^{(m)} \Vert \leq 1$ and $\Vert \Sigma_{i}^{(m)} \Vert_2 \leq {C_\eta}{m}^{-1}$ for all $i=1,\ldots,n$ and $m\geq M$, where $C_\eta$ is a universal constant (independent of $m$). Then, 
\begin{equation}
    \Vert \widetilde{W}^{(d),(m)} - W^{(d),(m)} \Vert_F = \mathcal{O}_p (m^{-1/2}). \label{eq: main result}
\end{equation}
\end{thm}
In other words, under the conditions in Theorem~\ref{thm:Main theorem}, it follows that for any probability $p>0$ there exist a constant $C^{'}$ and an integer $M^{'}$ (both of which may depend on $n$, $p$, $\varepsilon$, and $C_\eta$) such that for all $m\geq M^{'}$ we have $\operatorname{Pr}\{\Vert \widetilde{W}^{(d),(m)} - W^{(d),(m)} \Vert_F > C^{'} m^{-1/2}\} \leq p$. The proof of Theorem~\ref{thm:Main theorem} is detailed in Appendix~\ref{sec:proof of main theorem}. For simplicity of the presentation, we omit the superscript ${(m)}$ from all quantities in the rest of this section, as it should be clear that all quantities associated with Theorem~\ref{thm:Main theorem} are sequences in the dimension $m$ (where $n$ and $\varepsilon$ are fixed).

We now provide some remarks on the conditions in Theorem~\ref{thm:Main theorem}. 
Evidently, the constant $1$ in the condition $\Vert \mathbf{x}_i \Vert \leq 1$ is arbitrary and can be replaced with any other constant (since $\widetilde{\mathbf{x}}_i$ can always be normalized appropriately). Additionally, note that even though the quantities $\Vert \Sigma_{i} \Vert_2$ are required to decrease with $m$, the expected noise magnitudes $\mathbb{E}\Vert \eta_i \Vert^2$ (which are equal to $\operatorname{Tr}\{\Sigma_i\}$) can remain constant, and can possibly be large compared to the magnitudes of the clean data points $\Vert \mathbf{x}_i \Vert^2$. For example, if we have $ \Sigma_{i} = m^{-1} I_m$ for all $i$, where $I_m$ is the $m\times m$ identity matrix, then it follows that $\mathbb{E}\Vert \eta_i \Vert^2 = \operatorname{Tr}\{\Sigma_i\} = 1$, asserting that the magnitude of the noise is greater or equal to that of the clean data points (under the condition $\Vert \mathbf{x}_i \Vert \leq 1$). In this regime of non-vanishing high-dimensional noise, the condition $\Vert \Sigma_{i} \Vert_2 \leq {C_\eta} m^{-1}$ guarantees that the noise spreads-out in Euclidean space, and does not concentrate too much in any particular direction (observe that $\Vert \Sigma_{i} \Vert_2$ is the largest singular value of $\Sigma_{i}$, and is therefore the variance of the noise in the direction where it is largest). Hence, the condition $\Vert \Sigma_i \Vert_2 \leq C_\eta m^{-1}$ is primarily a convenience for considering noise that has bounded magnitude regardless of the ambient dimension (since $\mathbb{E}\Vert \eta_i \Vert_2^2 \leq C_\eta$), and whose variance in any particular direction is not too large. In many situations, the data can be normalized appropriately to satisfy this condition, see Remark~\ref{remark:high-dimensional noise model} below and the discussion in Section~\ref{subsubsec:scRNA-seq model asymptotic discussion}.
Clearly, the setup of Theorem~\ref{thm:Main theorem} accommodates for heteroskedastic noise, and importantly, the ratios between the noise magnitudes $\mathbb{E}\Vert \eta_i \Vert^2$ for different data points can be arbitrary. 

The main reason behind the robustness of the doubly-stochastic normalization to high-dimensional heteroskedastic noise, is that it is invariant to the type of bias introduced by heteroskedastic noise. Specifically, our analysis in the proof of Theorem~\ref{thm:Main theorem} (see Appendix~\ref{sec:proof of main theorem}) shows that for $i\neq j$,
\begin{equation}
    \Vert \widetilde{\mathbf{x}}_i - \widetilde{\mathbf{x}}_j \Vert^2 \overset{p}{\longrightarrow} \mathbb{E} \Vert \widetilde{\mathbf{x}}_i - \widetilde{\mathbf{x}}_j \Vert^2 = \mathbb{E}\Vert \eta_i \Vert^2 + \Vert {x}_i - {x}_j \Vert^2 + \mathbb{E}\Vert \eta_j \Vert^2, \label{eq:noisy distances bias}
\end{equation}
where $\overset{p}{\longrightarrow}$ stands for convergence in probability, and correspondingly,
\begin{equation}
    \widetilde{K}_{i,j} \overset{p}{\longrightarrow} \operatorname{exp}( - \mathbb{E}{\Vert {\eta}_i\Vert^2}/{\varepsilon}) \cdot K_{i,j} \cdot \operatorname{exp}( - \mathbb{E}{\Vert {\eta}_j \Vert^2}/{\varepsilon}), \label{eq:K_tilde diagonal bias}
\end{equation}
for all $i,j$ (since $\widetilde{K}_{i,i} = K_{i,i} = 0)$.
Crucially, $\widetilde{K}$ in~\eqref{eq:K_tilde diagonal bias} is biased by symmetric diagonal scaling, which is precisely the type of bias corrected automatically by the doubly-stochastic normalization~\eqref{eq:W_d def}.

Equations~\eqref{eq:noisy distances bias} and~\eqref{eq:K_tilde diagonal bias} also highlight why zeroing-out the main diagonal of the Gaussian kernel (see Eq.~\eqref{eq:K def}) is important. Without it, the entries on the main diagonal of $\widetilde{K}$ would be $1$, while the off-diagonal entries of $\widetilde{K}$ would be small due to the bias in the noisy pairwise distances~\eqref{eq:noisy distances bias}. Thus, $\widetilde{K}$ would be close to the identity matrix, which would render any normalization (row-stochastic, symmetric, or doubly-stochastic) ineffective. 

\begin{remark} \label{remark:high-dimensional noise model}
   Consider an alternative setting for high-dimensionality where the noise is only required to have bounded variance in each direction, i.e., $\Vert \Sigma_i \Vert_2 \leq C_\eta$ for some universal constant $C_\eta$. This assumption holds, for instance, in the standard model where $\eta_{i}$ has bounded variance in each coordinate and is uncorrelated between different coordinates. In addition, as the dimension $m$ increases, suppose that the newly-added clean data coordinates are determined by a latent variable that is sampled from some underlying distribution. Specifically, suppose that each clean observation $\mathbf{x}_i$ is given by
   \begin{equation}
        \mathbf{x}_{i} = [F_i(y_1),\ldots,F_i(y_m)]^T, 
   \end{equation}
   where $F_i$ is a bounded function, and $y_1,\ldots,y_m$ are i.i.d samples from some latent ``coordinate'' random variable $Y$ (which can be multivariate or reside in a non-Euclidean space). In this case, one has
   \begin{align}
       &\Vert \mathbf{x}_i \Vert^2 = \sum_{k=1}^m (F_i(y_k))^2 = m \left( \mathbb{E}_{y\sim Y} (F_i(y))^2 + \mathcal{O}_p(m^{-1/2}) \right), \\
       &\Vert \mathbf{x}_i - \mathbf{x}_j \Vert^2 =  \sum_{k=1}^m (F_i(y_k) - F_j(y_k) )^2 = m \left( \mathbb{E}_{y\sim Y} |F_i(y) - F_j(y) |^2 + \mathcal{O}_p(m^{-1/2}) \right),
   \end{align}
   which is due to Hoeffding's inequality~\cite{hoeffding1994probability} (for sums of independent and bounded random variables). 
   Evidently, a natural distance between $\mathbf{x}_i$ and $\mathbf{x}_j$ in this setting is $\mathbb{E}_{y\sim Y} |F_i(y) - F_j(y) |^2$ as it does not depend on the ambient dimension $m$ and allows for a constant kernel parameter $\varepsilon$ to be used for all $m$.
   This suggests that the noisy observations $\widetilde{\mathbf{x}}_i$ should be normalized by $\sqrt{m}$, which places us in the setting of Theorem~\ref{thm:Main theorem} since $\Vert \mathbf{x}_i/\sqrt{m} \Vert^2 = \mathbb{E}_{y\sim Y} (F_i(y))^2 + \mathcal{O}_p(m^{-1/2}) = \mathcal{O}_p(1)$, and $\Vert \Sigma_i/{m} \Vert_2 \leq C_\eta {m}^{-1}$.
\end{remark}

\section{Examples} \label{sec:numerical example}
\subsection{Example 1: The unit circle embedded in high-dimensional space}
In our first example, we sampled $n=10^3$ points uniformly from the unit circle in $\mathbb{R}^2$, and embedded them in $\mathbb{R}^m$, for $m\in [10,10^4]$, using randomly-generated orthogonal transformations. In more details, we first sampled angles $\theta_1,\ldots,\theta_n$ independently and uniformly from $[0,2\pi]$. Then, for each embedding dimension $m$, we generated a random orthogonal matrix $R_m\in \mathbb{R}^{m \times 2}$ (i.e., such that $R_m^T R_m = I_m$), and computed the data points $\{\mathbf{x}_i\}$ as
\begin{equation}
\mathbf{x}_i =  R_m \cdot 
\begin{bmatrix}
\cos(\theta_i) \\
\sin(\theta_i)
\end{bmatrix},
\qquad i = 1, \ldots, n. \label{eq:x_i circle def}
\end{equation}
Note that as a result, the magnitude of all points is constant, with $\Vert \mathbf{x}_i \Vert = 1$ for all $1\leq i \leq n$ and embedding dimension $m$.

\subsubsection{Gaussian noise with arbitrary variances} \label{subsubsec:circle example 1}
We begin by demonstrating Theorem~\ref{thm:Main theorem} numerically. Towards that end, we created the noise as follows. For every embedding dimension $m$, we set $\Sigma_i = \operatorname{diag}([\sigma^2_{i,1},\ldots,\sigma^2_{i,m}])$ (so that the noise is uncorrelated between coordinates), and generated the noise standard-deviations $\sigma_{i,j}$ according to
\begin{equation}
    \sigma_{i,j} = \sqrt{\frac{\alpha_i \beta_j}{m}}, \label{eq:noise std sim}
\end{equation}
where $\{\alpha_i\}_{i=1}^n$, $\{\beta_j\}_{j=1}^m$ were sampled (independently) from the uniform distribution over $[0.05,0.5]$. Therefore, the noise magnitudes $\mathbb{E}\Vert \eta_{i}\Vert ^2$ satisfy
\begin{equation}
    \frac{1}{400} \leq \mathbb{E} \Vert \eta_{i}\Vert ^2 \leq \frac{1}{4}, \label{eq:noise variance sim bounds}
\end{equation}
for all $1\leq i \leq n$, and can take any values in that range. Importantly, the noise magnitudes can vary substantially between data points, which is key in our setting.
Then, $\{\eta_{i}[j]\}_{i,j}$ were sampled (independently) according to
\begin{equation}
    \eta_{i}[j] \sim \mathcal{N} (0,\sigma_{i,j}^2), \qquad i=1,\ldots,n, \quad j = 1,\ldots,m, \label{eq:noise sim sampling}
\end{equation}
where $\eta_{i}[j]$ stands for the $j$'th entry of $\eta_i$.
Once we generated the noisy data points $\widetilde{\mathbf{x}}_1,\ldots,\widetilde{\mathbf{x}}_n$ according to~\eqref{eq:x_tilde def}, we formed the clean and noisy kernel matrices $K$ and $\widetilde{K}$ with $\varepsilon=0.1$, and computed $W^{(d)}$, $\widetilde{W}^{(d)}$ using Algorithm~\ref{alg:SK sym} with $\delta=10^{-12}$. Last, we also evaluated $W^{(r)}$, $W^{(s)}$ and $\widetilde{W}^{(r)}$, $\widetilde{W}^{(s)}$ using $K$ and $\widetilde{K}$, respectively, according to~\eqref{eq:W_r def} and~\eqref{eq:W_s def}.

The behavior of the errors $\Vert \widetilde{W}^{(d)} - W^{(d)} \Vert_F^2$, $\Vert \widetilde{W}^{(r)} - W^{(r)} \Vert_F^2$, $\Vert \widetilde{W}^{(s)} - W^{(s)} \Vert_F^2$ as a function of $m$ can be seen in Figure~\ref{fig:err vs m}. It is evident that for $m>100$ the error for the doubly-stochastic normalization is substantially smaller than that for the row-stochastic normalization or for the symmetric normalization. Additionally, the error for the doubly-stochastic normalization decreases linearly in logarithmic scale, while the errors for the row-stochastic and the symmetric normalizations reach saturation and never fall below a certain value. In this experiment, the slope of $\log (\Vert \widetilde{W}^{(d)} - W^{(d)} \Vert_F^2)$ versus $\log m$ (between $m=10^2$ and $m=10^4$) was $-0.9996$, matching the slope suggested by the upper bound in Theorem~\ref{thm:Main theorem} (which implies a slope of $-1$ for the squared Frobenius norm).

\begin{figure} 
  \centering
  	{
    \includegraphics[width=0.5\textwidth]{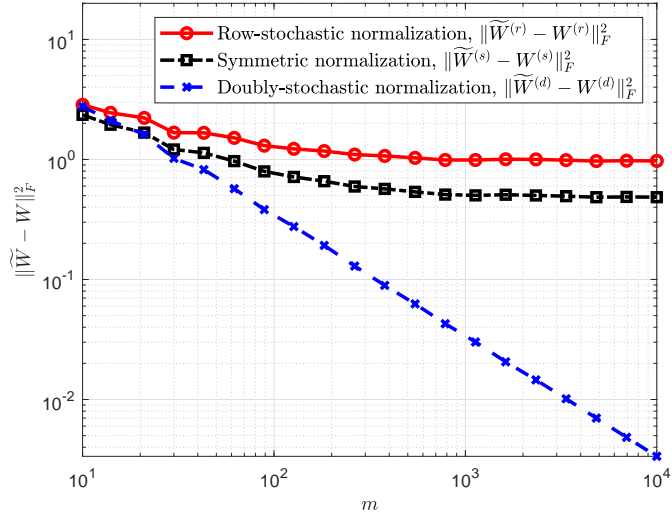} 
    }
    \caption
    {Squared Frobenius loss (averaged over $10$ trials) between clean and noisy affinity matrices from different normalizations, versus the dimension $m$. The dataset is the unit circle embedded in different dimensions (see~\eqref{eq:x_i circle def}), with $n=10^3$ and heteroskedastic noise simulated according to~\eqref{eq:noise std sim}--\eqref{eq:noise sim sampling}} \label{fig:err vs m}
    \end{figure} 
    
In Figure~\ref{fig:W_tilde vs W} we depict the noisy affinities $\widetilde{W}^{(d)}_{i,j}$, $\widetilde{W}^{(r)}_{i,j}$, $\widetilde{W}^{(s)}_{i,j}$ versus their corresponding clean affinities $W^{(d)}_{i,j}$, $W^{(r)}_{i,j}$, $W^{(s)}_{i,j}$, for $m=10^4$. It can be observed that the noisy affinities from the doubly-stochastic normalization concentrate near their corresponding clean affinities, while the noisy affinities from the row-stochastic and symmetric normalizations deviate substantially from their clean counterparts, particularly for larger affinity values.

\begin{figure} 
  \centering
  	{
  	\subfloat[Doubly-stochastic normalization~\eqref{eq:W_d def}]  
  	{
    \includegraphics[width=0.33\textwidth]{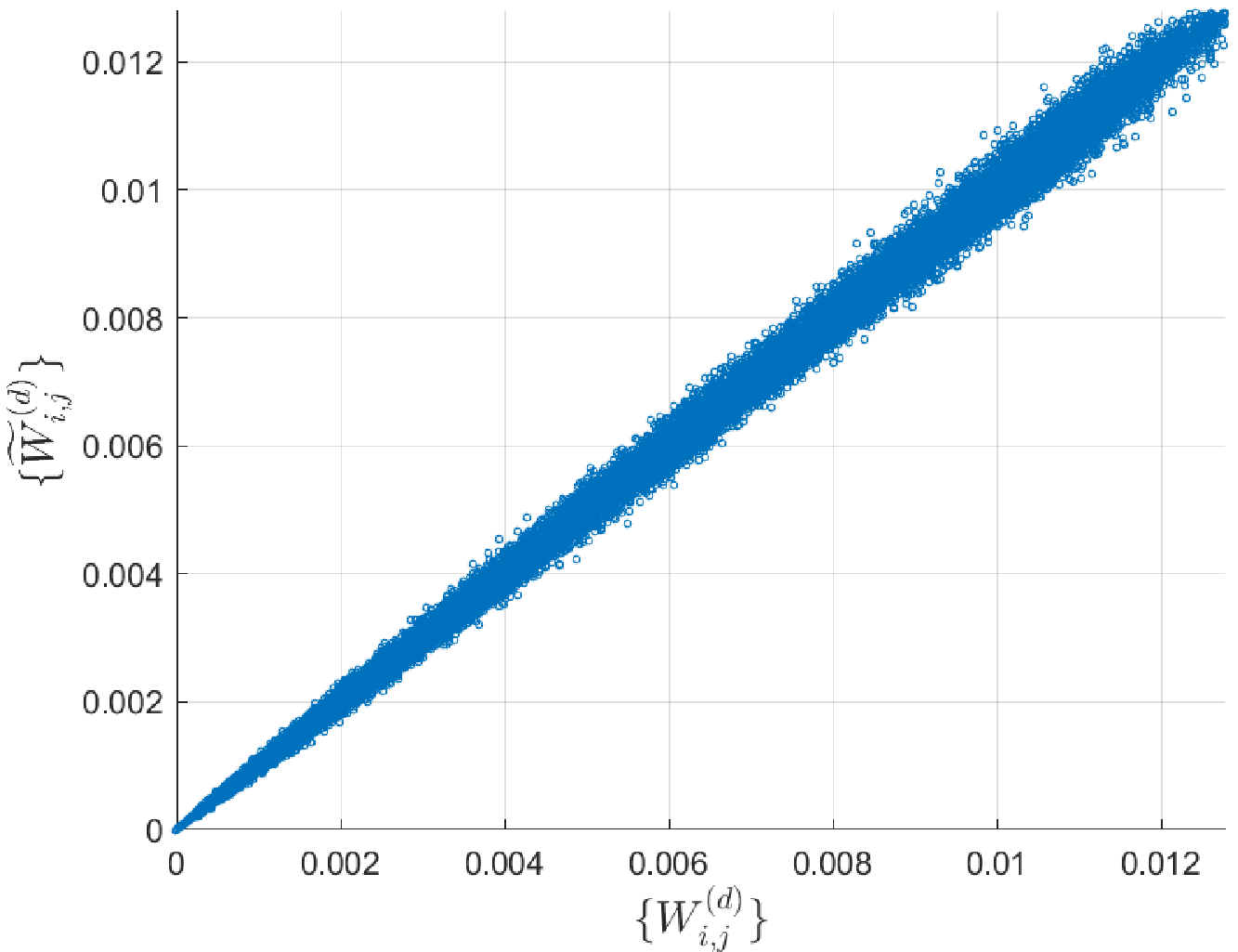} 
    }
    \subfloat[Row-stochastic normalization~\eqref{eq:W_r def}]  
  	{
    \includegraphics[width=0.33\textwidth]{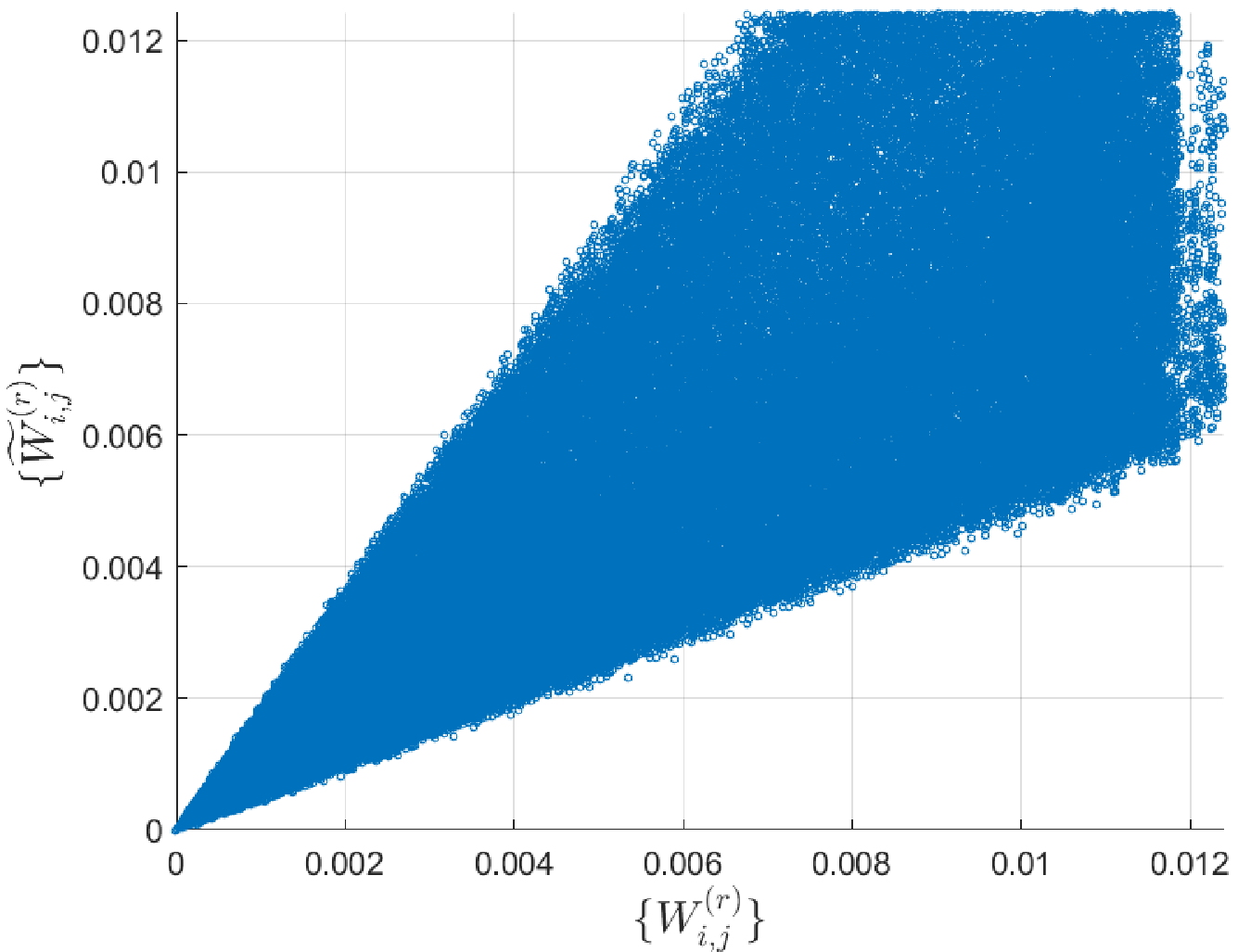}
    }
    \subfloat[Symmetric normalization~\eqref{eq:W_s def}]  
  	{
    \includegraphics[width=0.33\textwidth]{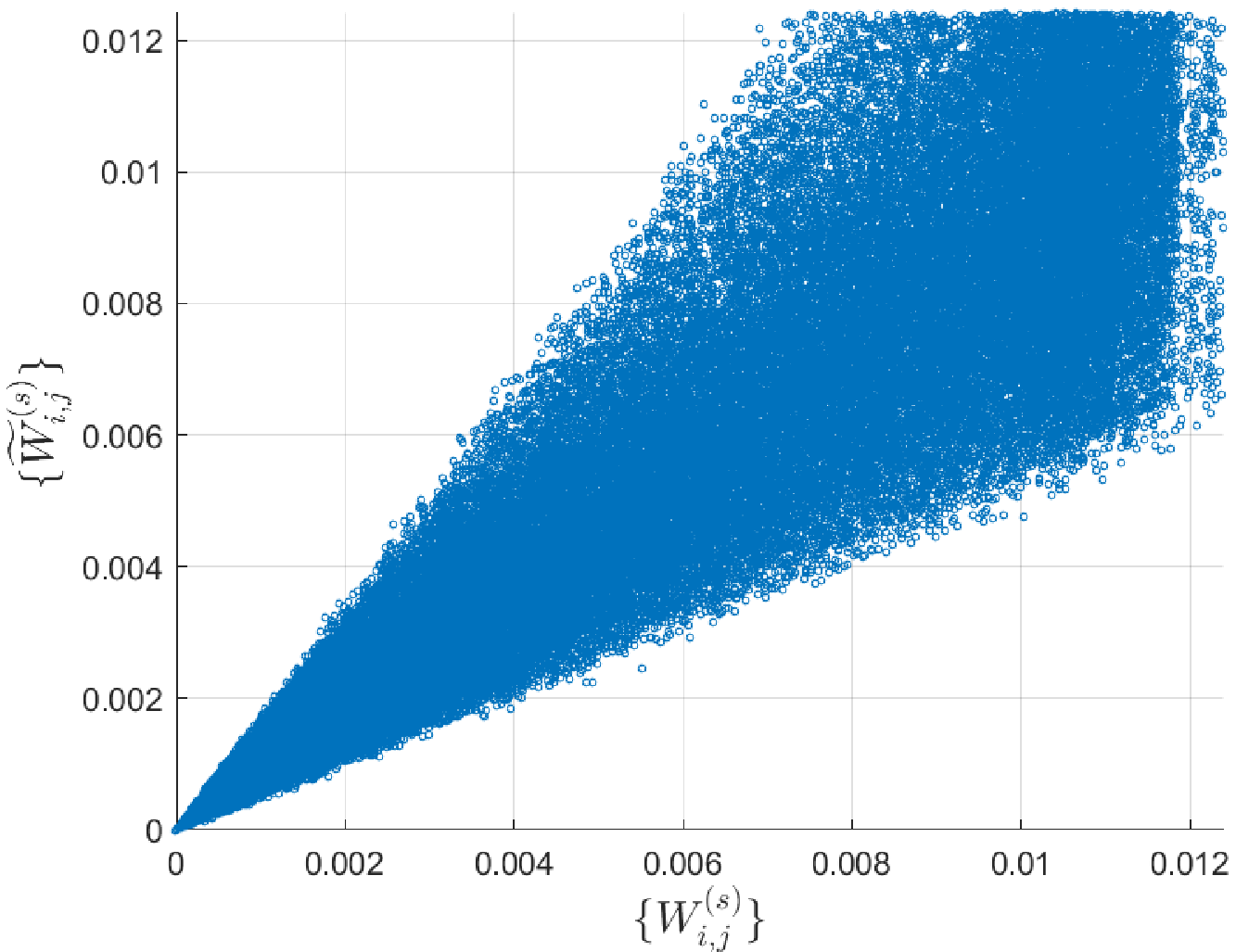}
    }
    }
    \caption
    {Entries of the noisy affinity matrices (y-axis) versus the corresponding entries in the clean affinity matrices (x-axis), using different normalizations. The dataset is the unit circle (see~\eqref{eq:x_i circle def}), with $n=10^3$, $m=10^4$, and heteroskedastic noise simulated according to~\eqref{eq:noise std sim}--\eqref{eq:noise sim sampling}.} \label{fig:W_tilde vs W}
    \end{figure} 

Last, in Figure~\ref{fig:W_tilde vs W heat map} we visually demonstrate the first row of the clean and noisy affinity matrices $W^{(d)}$, $W^{(r)}$, $W^{(s)}$ and $\widetilde{W}^{(d)}$, $\widetilde{W}^{(r)}$, $\widetilde{W}^{(s)}$, using $m=10^4$. Note that we only display about a quarter of all the entries, since all the other entries are vanishingly small.
It can be seen that the clean row-stochastic, clean symmetric, and clean doubly-stochastic affinities are all very similar, and resemble a Gaussian.  This is explained by the fact that both $W^{(d)}$ and $W^{(r)}$ are expected to approximate the heat kernel on the unit circle (see~\cite{coifman2006diffusionMaps,marshall2019manifold} and other related references given in the introduction), which is close to the Gaussian kernel with geodesic distance (for sufficiently small $\varepsilon$). Additionally, since the sampling density on the circle is uniform, $\operatorname{diag}(\mathbf{r})$ (from~\eqref{eq:W_r def}) is close to a multiple of the identity, and hence $W^{(s)}$ is expected to be close to $W^{(r)}$ (recall that $W^{(s)} = [\operatorname{diag}(\textbf{r})]^{-1/2} W^{(r)} [\operatorname{diag}(\textbf{r})]^{1/2}$). Indeed, we found that $\Vert W^{(d)} - W^{(r)} \Vert_F^2 \approx \Vert W^{(d)} - W^{(s)} \Vert_F^2 \approx 0.01$.

Importantly, the doubly-stochastic normalization recovers the true affinities with high accuracy, with an almost perfect match between the corresponding clean and noisy affinities. On the other hand, there is an evident discrepancy between the corresponding clean and noisy affinities from the row-stochastic normalization and from the symmetric normalization.

\begin{figure} 
  \centering
  	{
  	\subfloat[Doubly-stochastic normalization~\eqref{eq:W_d def}]  
  	{
    \includegraphics[width=0.33\textwidth]{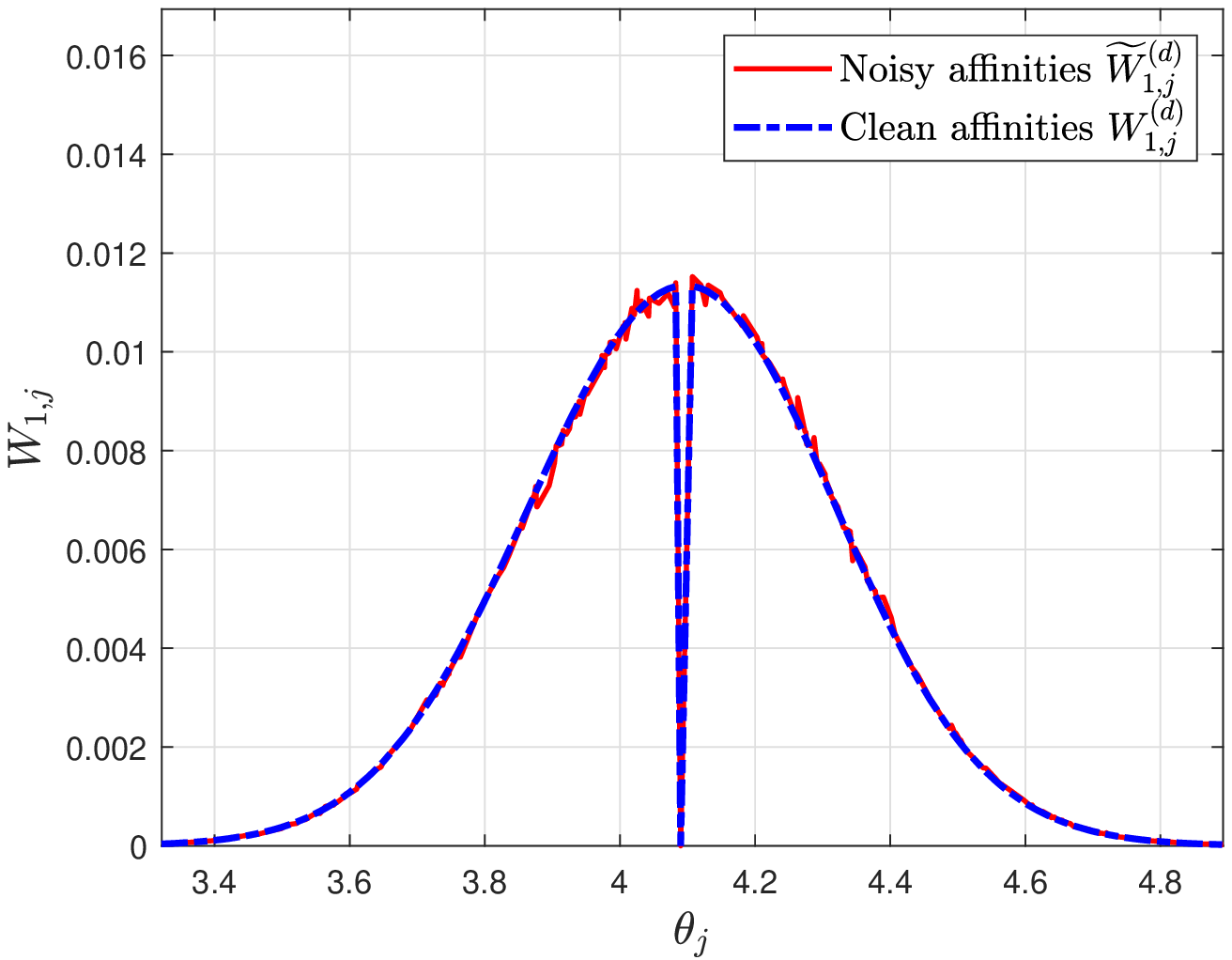} 
    }
    \subfloat[Row-stochastic normalization~\eqref{eq:W_r def}]  
  	{
    \includegraphics[width=0.33\textwidth]{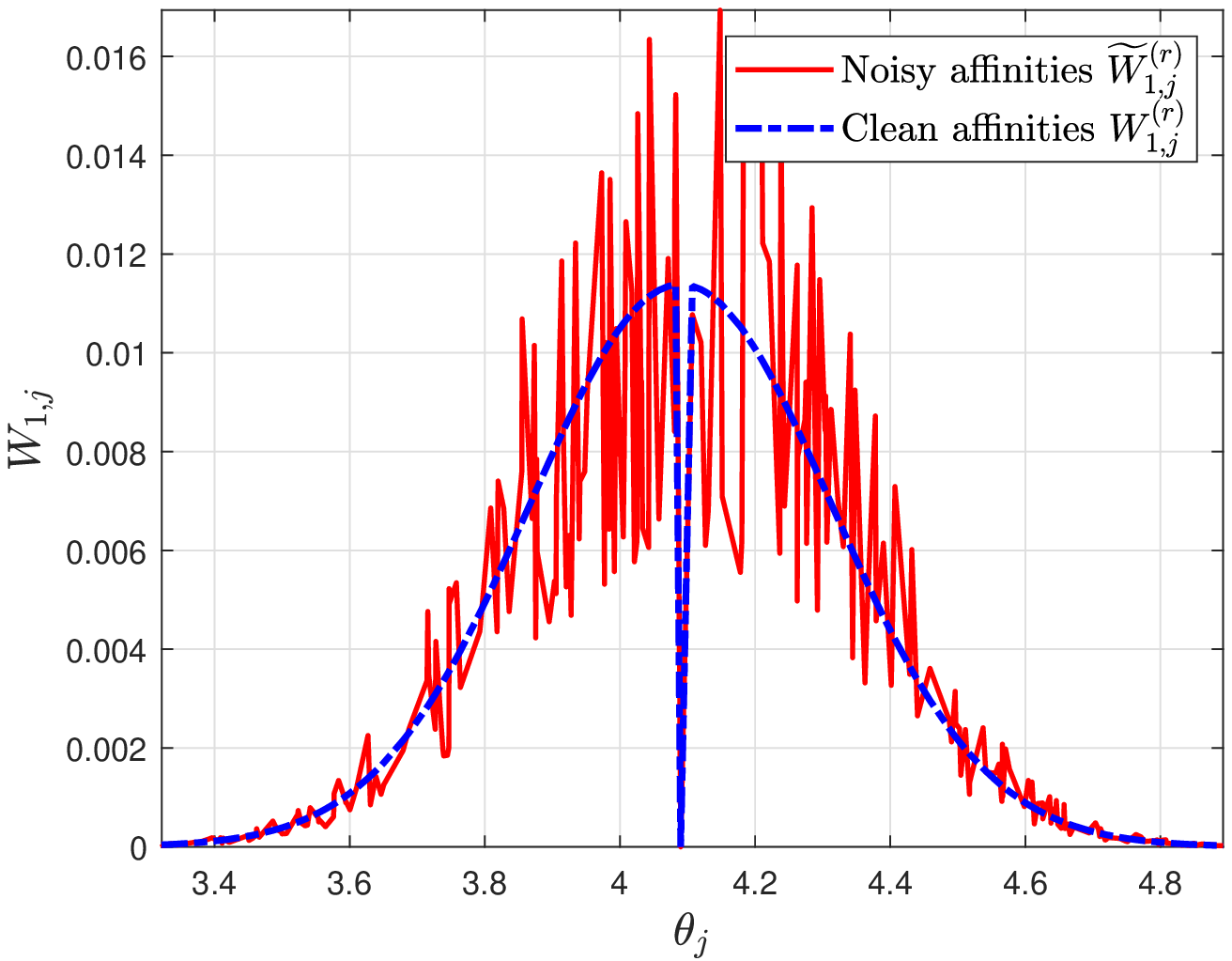}
    }
    \subfloat[Symmetric normalization~\eqref{eq:W_s def}]  
  	{
    \includegraphics[width=0.33\textwidth]{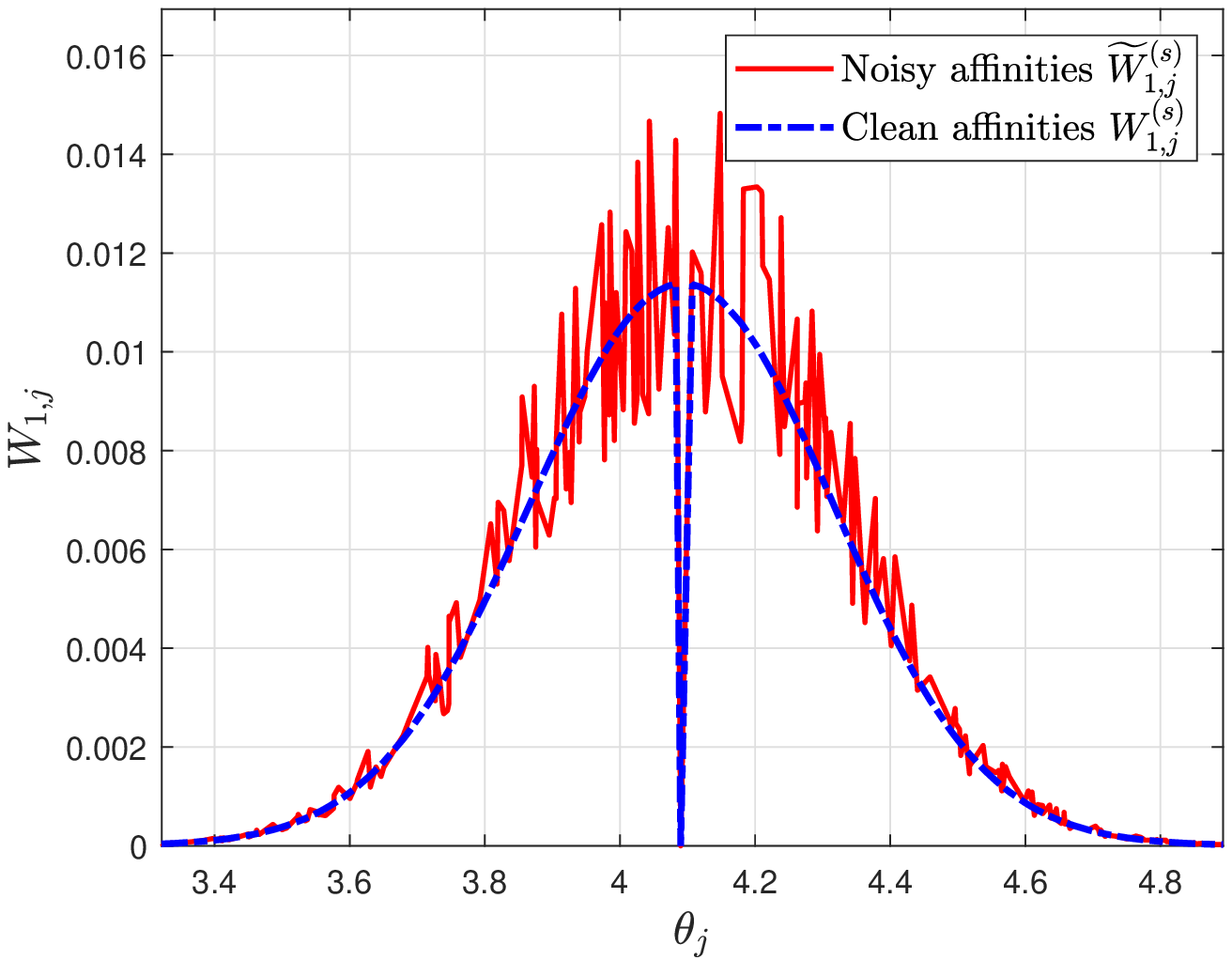}
    }
    }
    \caption
    {First row of the clean and noisy affinity matrices obtained using different normalizations. The dataset is the unit circle (see~\eqref{eq:x_i circle def}), with $n=10^3$, $m=10^4$, and heteroskedastic noise simulated according to~\eqref{eq:noise std sim}--\eqref{eq:noise sim sampling}.} \label{fig:W_tilde vs W heat map}
    \end{figure} 
    
\subsubsection{Noise sampled uniformly from a ball with smoothly varying radius} \label{subsubsec:circle example 2}
Next, we proceed by demonstrating the robustness of the leading eigenvectors from the doubly-stochastic normalization under heteroskedastic noise, and in particular, in the presence of noise whose magnitude depends on the local geometry of the clean data. Specifically, we simulated heteroskedastic noise whose magnitude varies smoothly according to the angle $\theta_i$ of each point $\mathbf{x}_i$ on the circle (see~\eqref{eq:x_i circle def}), according to
\begin{equation}
    \eta_i \sim U\left( \mathcal{B}_{\rho(\theta_i)}\right), \qquad \qquad \rho(\theta) = 0.01 + 0.99\frac{1+\cos(2\theta)}{2}, \label{eq:noise model uniform}
\end{equation}
where $U\left( \mathcal{B}_{r}\right)$ stands for the uniform distribution over $\mathcal{B}_{r}$, which is a ball with radius $r$ in $\mathbb{R}^m$ (centered at the origin). That is, every noisy observation $\widetilde{\mathbf{x}}_i$ is sampled uniformly from a ball whose center is $\mathbf{x}_i$ and its radius is $\rho(\theta_i)$ from~\eqref{eq:noise model uniform}.
Consequently, the maximal noise magnitude varies smoothly between $0.01$ (for $\theta=\pi/2, 3\pi/2$) and $1$ (for $\theta=0,\pi$).
A typical array of clean and noisy points arising from the noise model~\eqref{eq:noise model uniform} for dimension $m=2$ can be seen in Figure~\ref{fig:uniform noise varying radius 2d}.

\begin{figure} 
  \centering
  	{
    \includegraphics[width=0.4\textwidth]{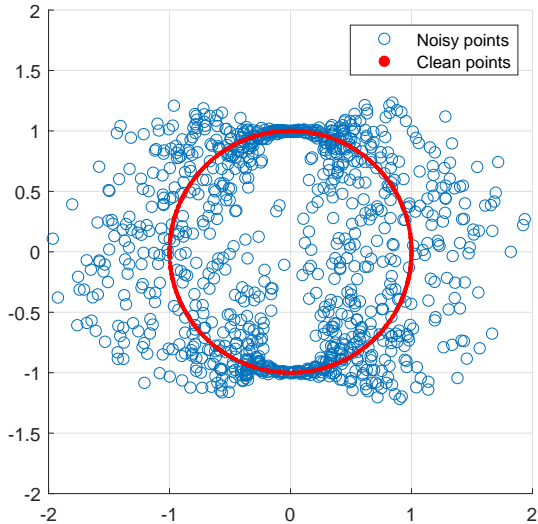} 
    }
    \caption
    {Typical array of clean and noisy data points for $n=1000$, $m=2$, and additive noise sampled uniformly from a sphere whose radius depends on the angle of the corresponding clean point (on the unit circle) according to~\eqref{eq:noise model uniform}.} \label{fig:uniform noise varying radius 2d}
\end{figure} 

We generated the noisy data points $\widetilde{\mathbf{x}}_1,\ldots,\widetilde{\mathbf{x}}_n$ according to~\eqref{eq:x_tilde def} for dimension $m=500$, and formed the noisy kernel matrix $\widetilde{K}$ with $\varepsilon=0.1$. We next computed $\widetilde{W}^{(d)}$ using Algorithm~\ref{alg:SK sym} with $\delta=10^{-12}$, and evaluated $\widetilde{W}^{(r)}$, $\widetilde{W}^{(s)}$ using $\widetilde{K}$ according to~\eqref{eq:W_r def} and~\eqref{eq:W_s def}. 

Figure~\ref{fig:noisy eigenvectors} displays the five leading (right) eigenvectors of ${W}^{(d)}$, ${W}^{(r)}$, ${W}^{(s)}$, denoted by $\{{\psi}_{k}^{(d)}\}_{k=1}^5$, $\{{\psi}_{k}^{(r)}\}_{k=1}^5$, $\{{\psi}_{k}^{(s)}\}_{k=1}^5$, respectively, and the five leading (right) eigenvectors of $\widetilde{W}^{(d)}$, $\widetilde{W}^{(r)}$, $\widetilde{W}^{(s)}$, denoted by $\{\widetilde{\psi}_{k}^{(d)}\}_{k=1}^5$, $\{\widetilde{\psi}_{k}^{(r)}\}_{k=1}^5$, $\{\widetilde{\psi}_{k}^{(s)}\}_{k=1}^5$, respectively.
It can be seen that the leading eigenvectors from the doubly-stochastic normalization are almost unaffected by the noise, and approximate sines and cosines, which are the eigenfunctions of the Laplace-Beltrami operator on the circle. As sines and cosines are advantageous for expanding periodic functions, it is natural to employ the eigenvectors of $\widetilde{W}^{(d)}$ for the purposes of regression, interpolation, and classification over the dataset. It is important to mention that other useful bases and frames can potentially be constructed from $\widetilde{W}^{(d)}$ (see~\cite{coifman2006diffusionWavelets,hammond2011wavelets}). On the other hand, the eigenvectors obtained from $\widetilde{W}^{(r)}$ and $\widetilde{W}^{(s)}$ are strongly biased due to the heteroskedastic noise, and exhibit undesired effects such as discontinuities and localization. Specifically, as evident from Figure~\ref{fig:noisy eigenvectors}, the leading eigenvectors of $\widetilde{W}^{(r)}$ are discontinuous at $\theta=0$ and $\theta=\pi$, and the leading eigenvectors of $\widetilde{W}^{(s)}$ are localized around $\theta = \pi/2$ and $\theta = 3\pi/2$ (i.e., their values are close to $0$ around $\theta = 0$ and $\theta = \pi$). Clearly, this behaviour of the leading eigenvectors of $\widetilde{W}^{(r)}$ and $\widetilde{W}^{(s)}$ does not reflect the geometry of the clean data, but rather the characteristics of the noise (since the noise variance is smallest at $\theta = \pi/2,3\pi/2$  and largest at $\theta = 0,\pi$). 
\begin{figure} 
  \centering
  	{
  	\subfloat[Clean Doubly-stochastic]  
  	{
    \includegraphics[width=0.33\textwidth]{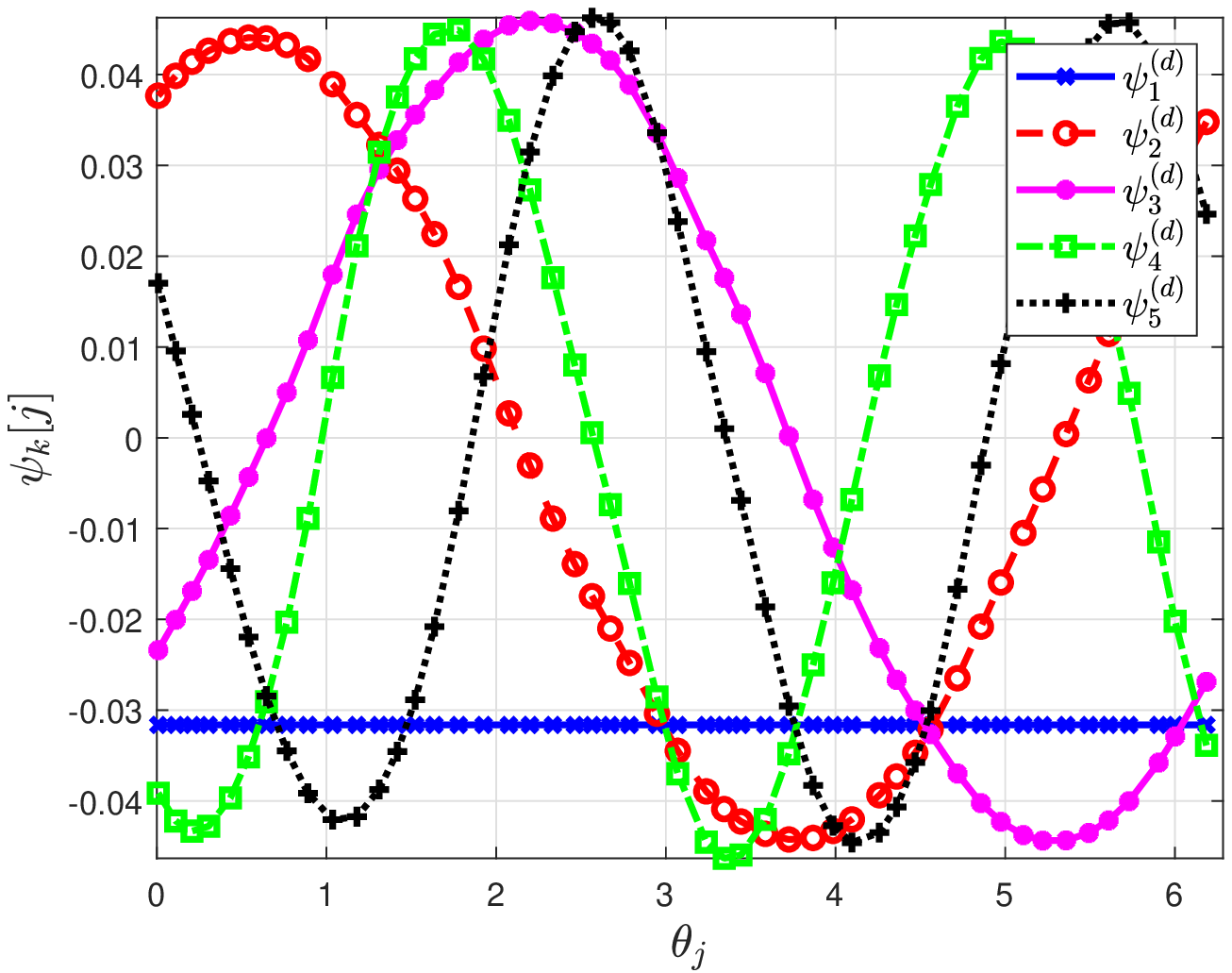} 
    }
    \subfloat[Clean Row-stochastic]  
  	{
    \includegraphics[width=0.33\textwidth]{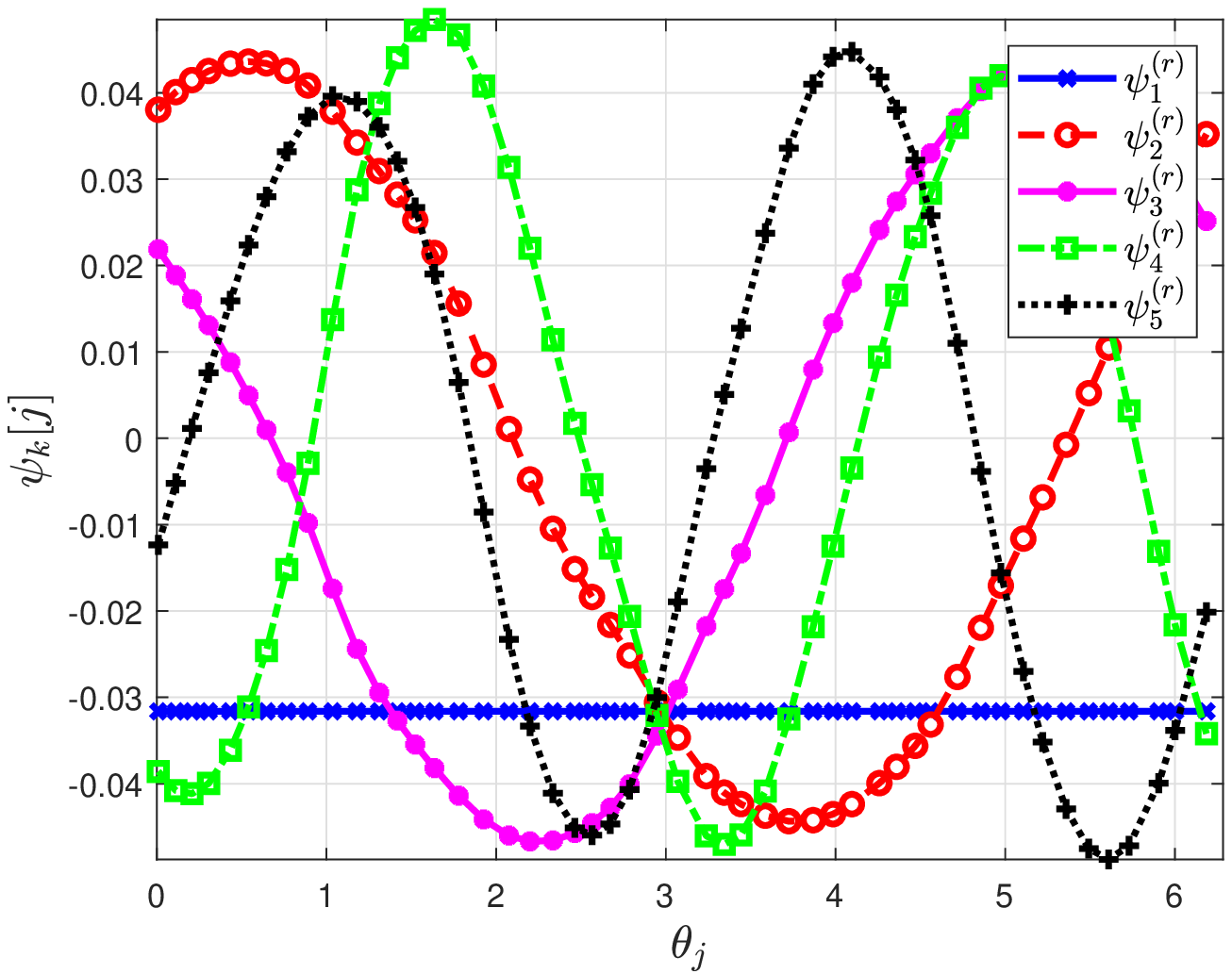}
    }
    \subfloat[Clean Symmetric]  
  	{
    \includegraphics[width=0.33\textwidth]{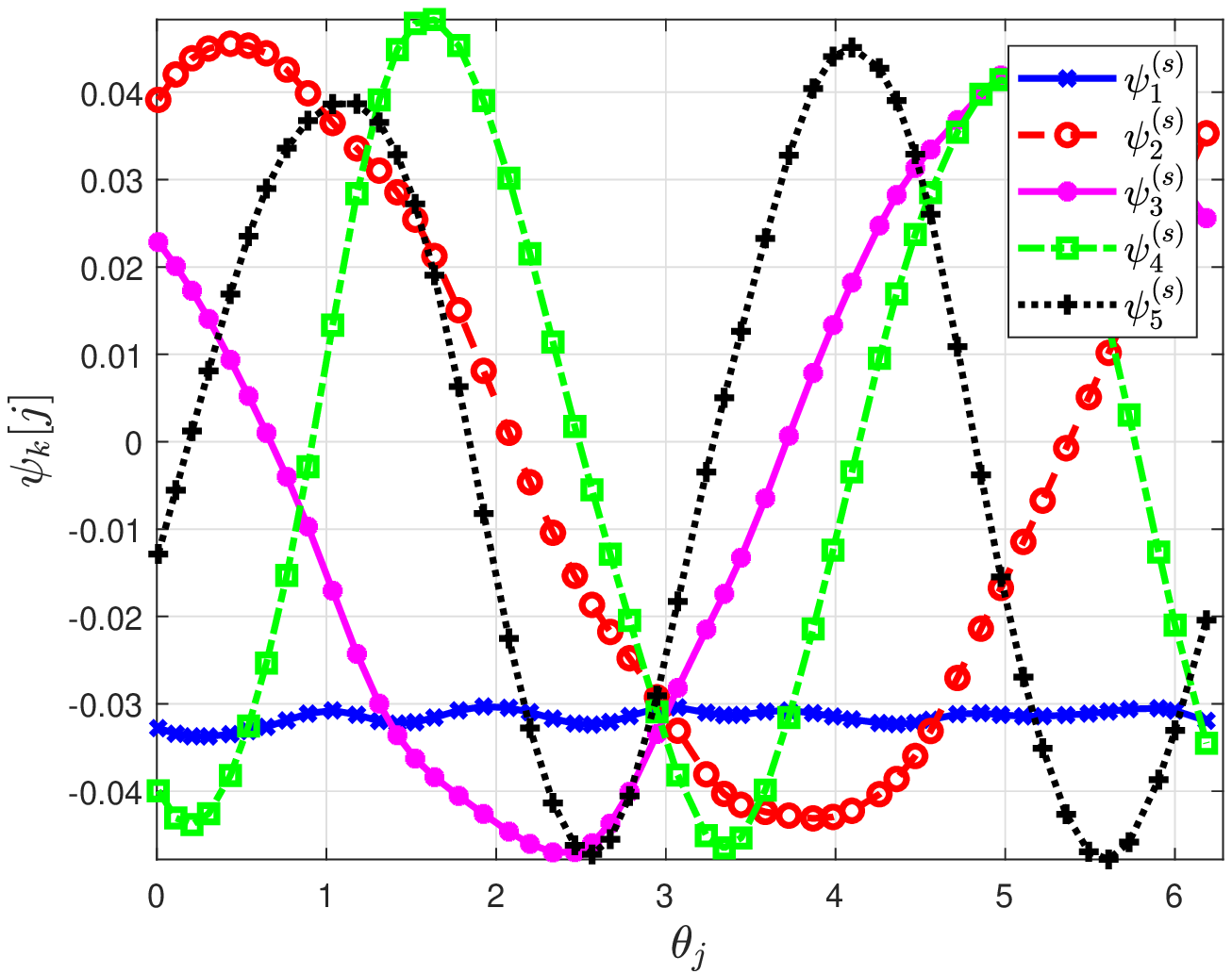}
    }\\
  	\subfloat[Noisy Doubly-stochastic]  
  	{
    \includegraphics[width=0.33\textwidth]{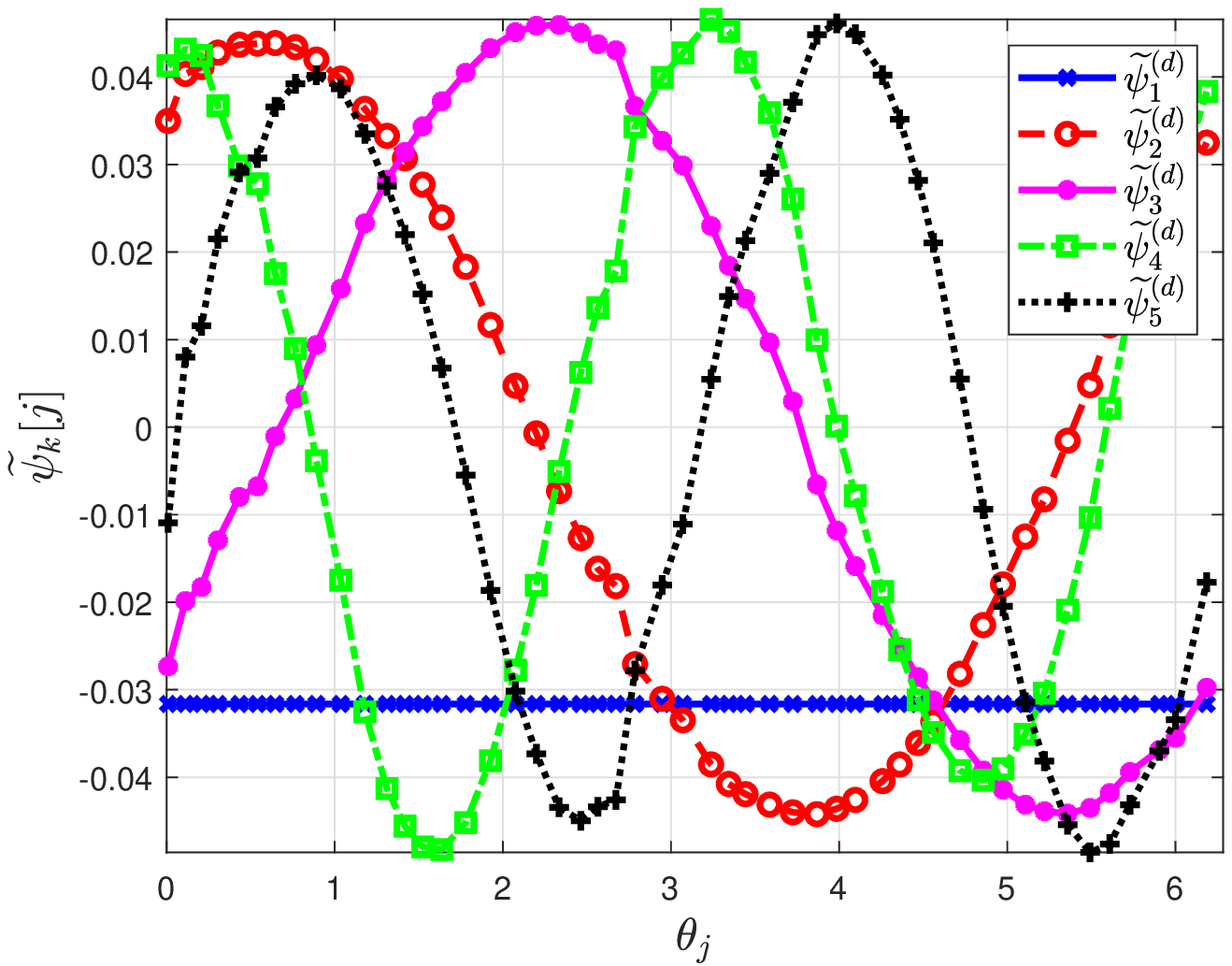} 
    }
    \subfloat[Noisy Row-stochastic]  
  	{
    \includegraphics[width=0.33\textwidth]{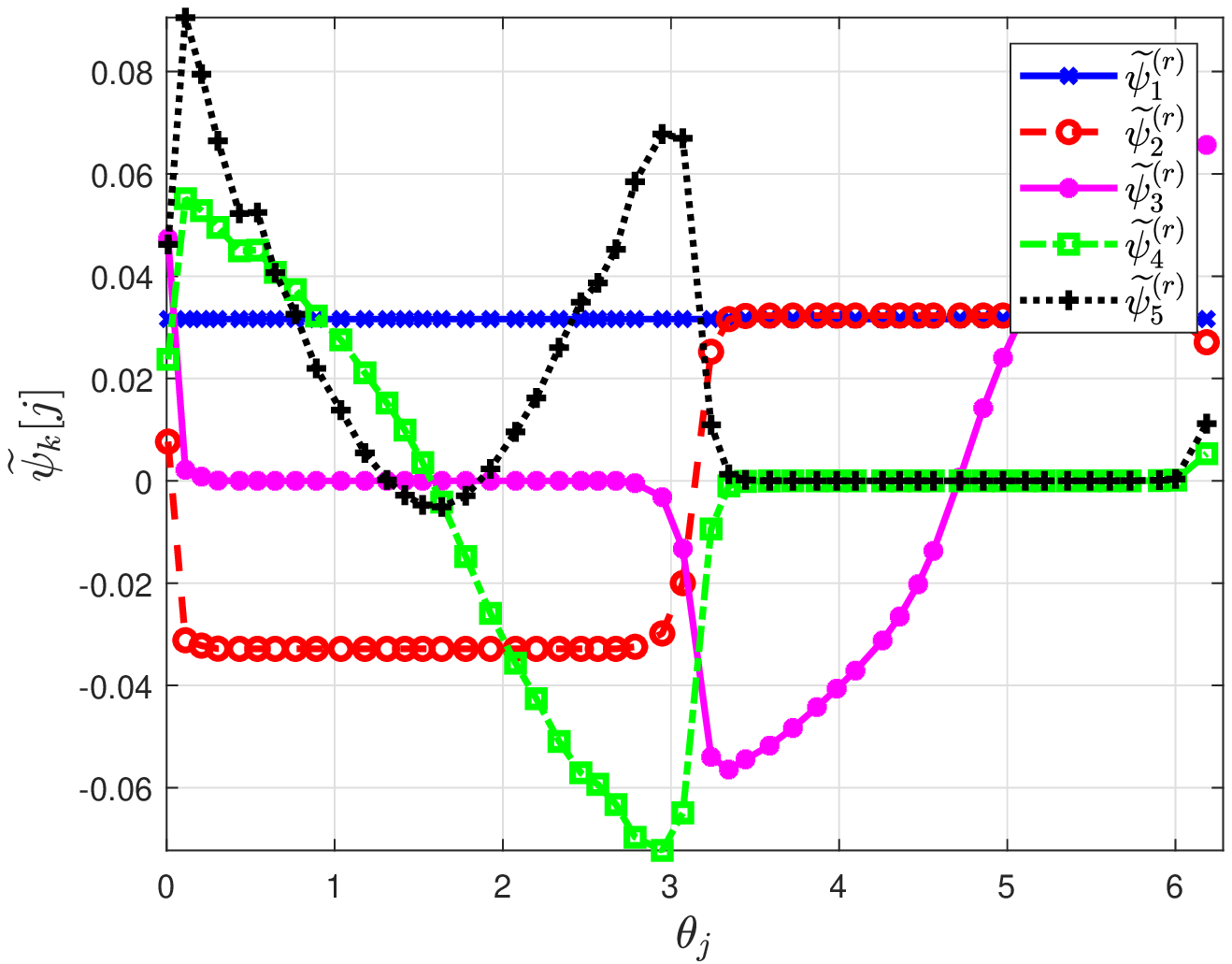}
    }
    \subfloat[Noisy Symmetric]  
  	{
    \includegraphics[width=0.33\textwidth]{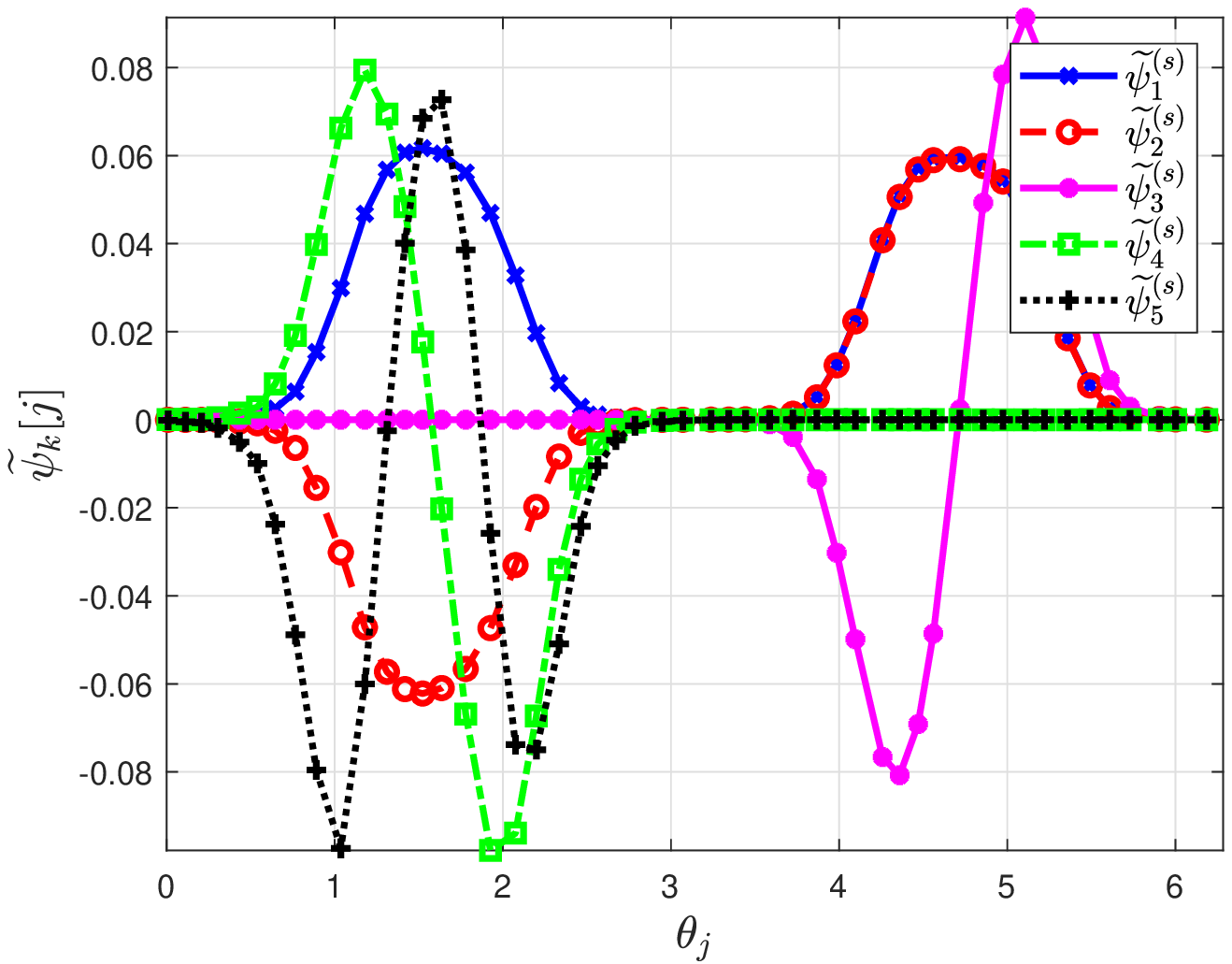}
    }
    }
    \caption
    {Eigenvectors corresponding to the five largest eigenvalues of the clean and noisy affinity matrices obtained from different normalizations. The top row corresponds to clean affinity matrices (from left to right: ${W}^{(d)}$, ${W}^{(r)}$, ${W}^{(s)}$) and the bottom row corresponds to noisy affinity matrices (from left to right: $\widetilde{W}^{(d)}$, $\widetilde{W}^{(r)}$, $\widetilde{W}^{(s)}$). The dataset is the unit circle, with $n=10^3$, $m=500$, and heteroskedastic noise generated according to~\eqref{eq:noise model uniform}.} \label{fig:noisy eigenvectors}
    \end{figure} 

In Figure~\ref{fig:embedding 2d} we illustrate the two-dimensional embedding of the noisy data points $\widetilde{\mathbf{x}}_1,\ldots,\widetilde{\mathbf{x}}_n$ using the second and third eigenvectors of $\widetilde{W}^{(d)}$, $\widetilde{W}^{(r)}$, and $\widetilde{W}^{(s)}$ (corresponding to their second- and third-largest eigenvalues). That is, the $x$-axis and $y$-axis values for each embedding are given by the entries of $\widetilde{\psi}^{(d)}_{2}$ and $\widetilde{\psi}^{(d)}_{3}$ for the doubly-stochastic normalization, $\widetilde{\psi}^{(r)}_2$ and $\widetilde{\psi}^{(r)}_3$ for the row-stochastic normalization, and $\widetilde{\psi}^{(s)}_2$ and $\widetilde{\psi}^{(s)}_3$ for the symmetric normalization (see also~\cite{belkin2003laplacian,coifman2006diffusionMaps}). 
It is clear that the embedding due to the doubly-stochastic normalization reliably represents the intrinsic structure of the clean dataset -- a unit circle with uniform density, whereas the embeddings due to the row-stochastic and the symmetric normalizations are incoherent with the geometry and density of the clean points.

\begin{figure} 
  \centering
  	{
  	\subfloat[Doubly-stochastic normalization~\eqref{eq:W_d def}]  
  	{
    \includegraphics[width=0.33\textwidth]{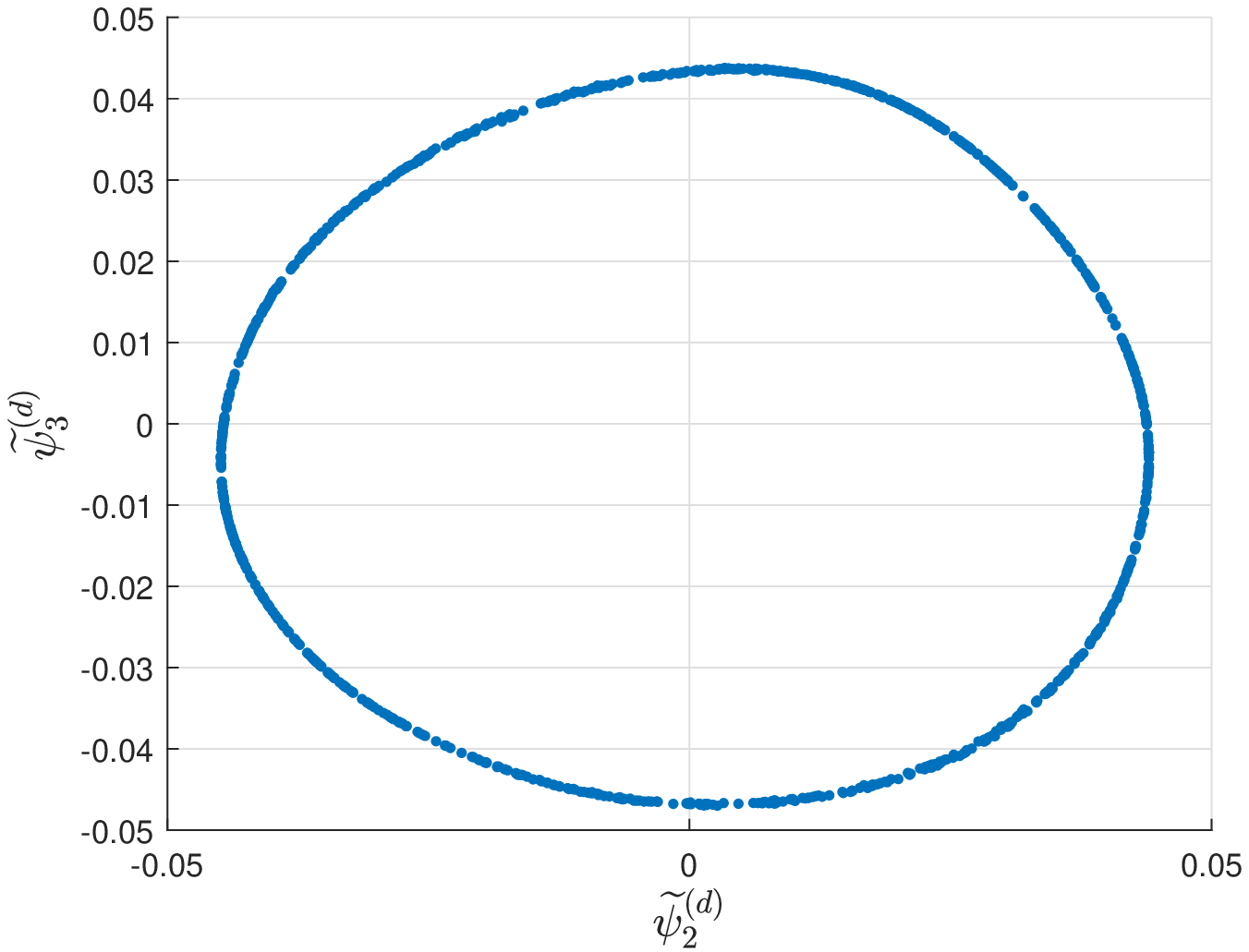} 
    }
    \subfloat[Row-stochastic normalization~\eqref{eq:W_r def}]  
  	{
    \includegraphics[width=0.33\textwidth]{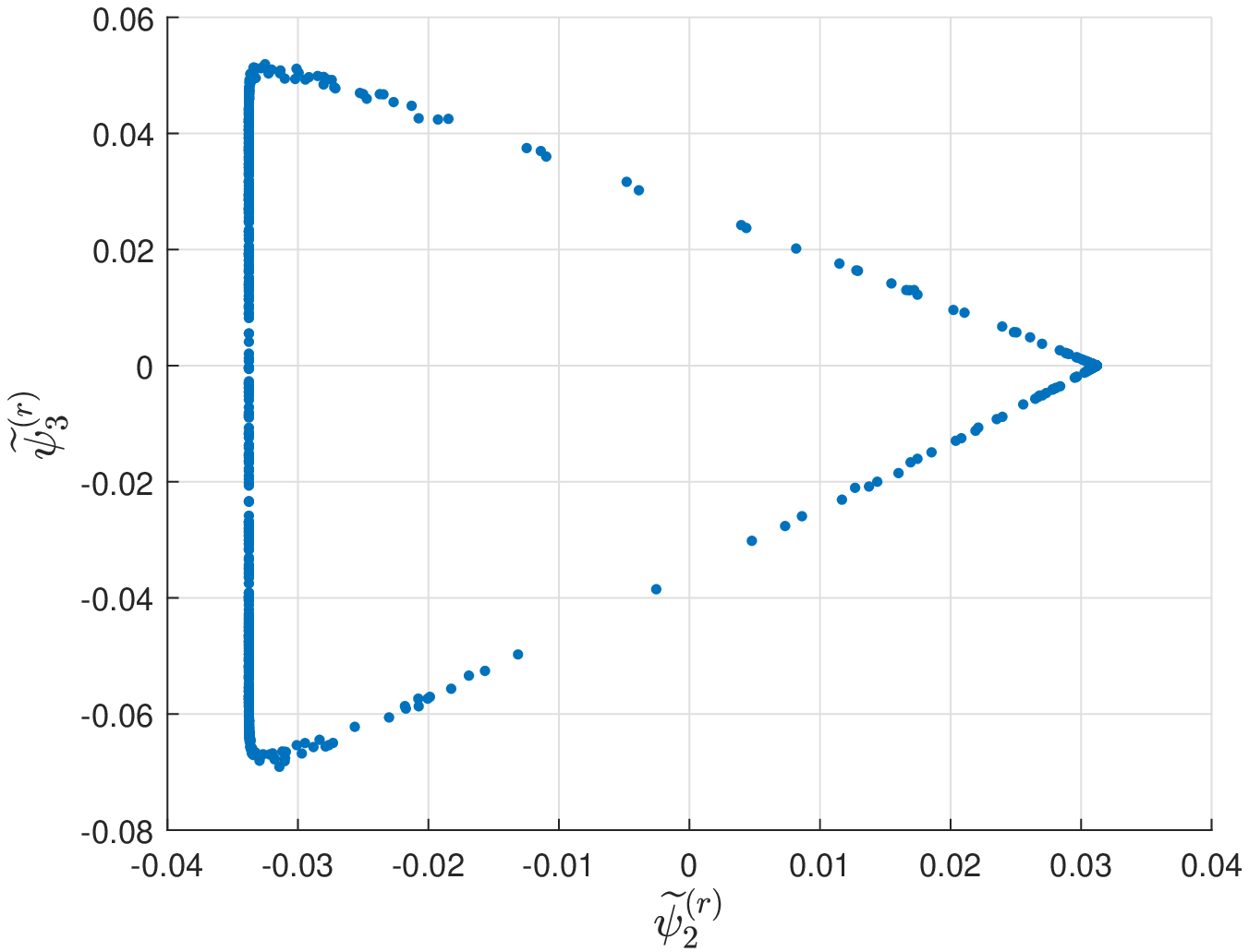}
    }
    \subfloat[Symmetric normalization~\eqref{eq:W_s def}]  
  	{
    \includegraphics[width=0.33\textwidth]{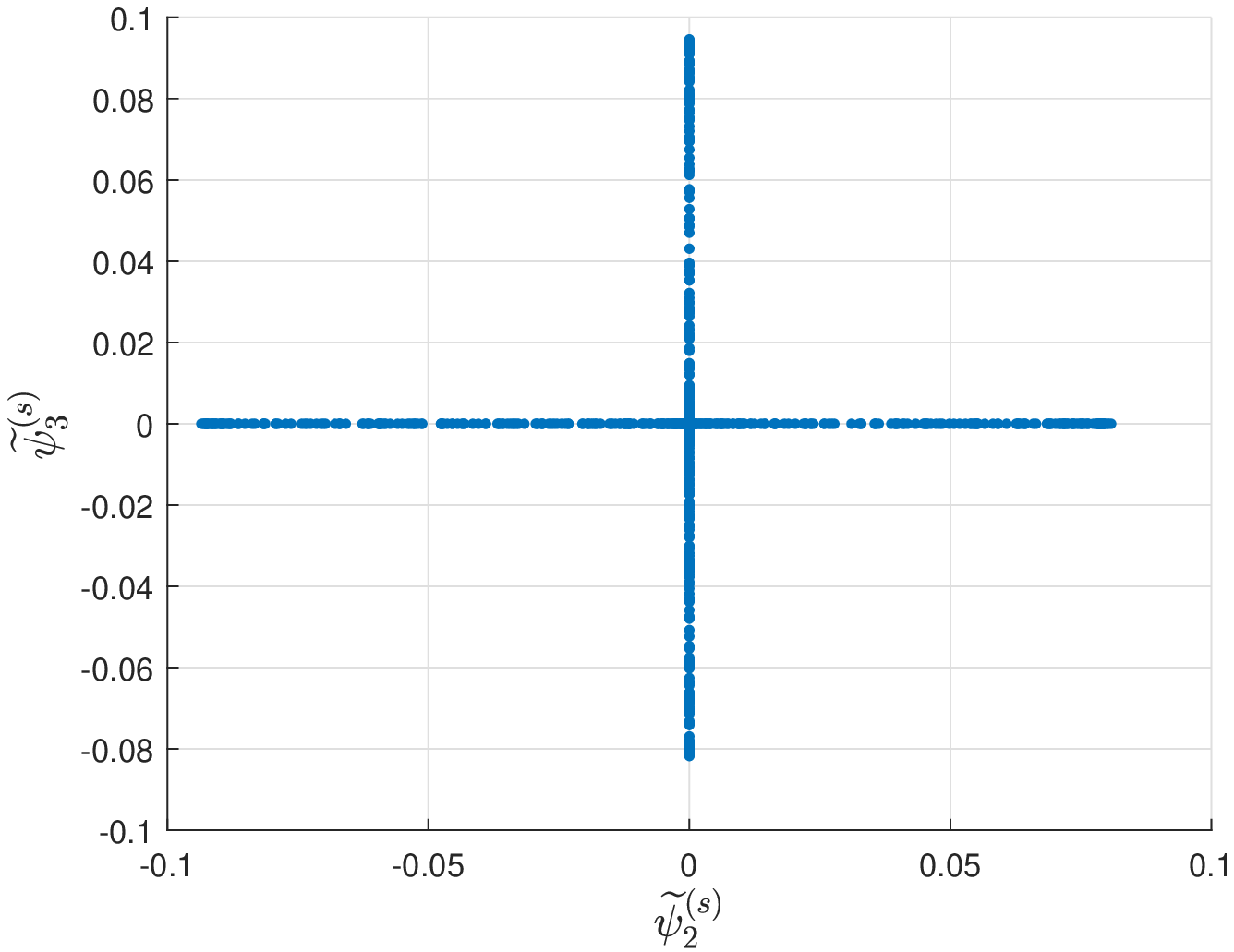}
    }
    }
    \caption
    {Two-dimensional embedding using the second and third eigenvectors (corresponding to the second- and third-largest eigenvalues) of affinity matrices obtained from different normalizations. The dataset is the unit circle, with $n=10^3$, $m=500$, and heteroskedastic noise generated according to~\eqref{eq:noise model uniform}.} \label{fig:embedding 2d}
    \end{figure} 
    
    \subsection{Example 2: Single-cell RNA sequence data} \label{subsec:SCRNAseq example}
Single-cell RNA sequencing (scRNA-seq) is a revolutionary technique for measuring target gene expressions of individual cells in large and heterogeneous samples~\cite{tang2009mrna,macosko2015highly}. Due to the method's high resolution (single-cell level) it allows for the discovery of rare cell populations, which is of paramount importance in immunology and developmental biology. A typical scRNA-seq dataset is an $m\times n$ nonnegative matrix corresponding to $n$ cells and $m$ genes, where its $(i,j)$'th entry is an integer called the \textit{read count}, describing the expression level of $i$'th gene in the $j$'th cell. Importantly, the total number of read counts (or in short \textit{total reads}) per cell (i.e., column sums) may vary substantially within a sample~\cite{kim2020demystifying}. 
We next exemplify the advantage of using the doubly-stochastic normalization for exploratory analysis of scRNA-seq data. 

\subsubsection{Simulated data} \label{subsubsec:scRNA-seq simulated data}
We begin with a simple prototypical example where the gene expression levels of cells are measured in two different batches, such that the number of total reads (per cell) within each batch is constant, but is substantially different between the batches. Therefore, the noise variance (modeled by the variance of the multinomial distribution, to be described shortly) differs between the observations in the two batches, giving rise to heteroskedastic noise. Such a scenario can arise naturally in scRNA-seq, either from the intrinsic read count variability common to such datasets, or when two datasets from two independent experiments are merged for unified analysis. 

We consider a simulated dataset which includes only two cell types, denoted by $\mathbf{p}_1,\mathbf{p}_2\in\mathbb{R}_{+}^{m}$, with $m=4000$ genes. The prototypes $\mathbf{p}_1$ and $\mathbf{p}_2$ were created by first sampling their entries uniformly (and independently) from $[0,1]$, and then normalizing them so that they sum to $1$. That is,
\begin{equation}
    \mathbf{p}_{1}[j] = \frac{\mathbf{z}_{1}[j]}{\sum_{k=1}^m \mathbf{z}_{1}[k]}, \qquad \mathbf{p}_{2}[j] = \frac{\mathbf{z}_{2}[j]}{\sum_{k=1}^m \mathbf{z}_{2}[k]}, \qquad \mathbf{z}_{1}[j],\; \mathbf{z}_{2}[j] \sim U[0,1], \qquad j=1,\ldots,m.  \label{eq:scRNAseq data gen 1}
\end{equation}
Next, each noisy observation $\widetilde{\mathbf{x}}_i$ was drawn from a multinomial distribution using either $\mathbf{p}_1$ or $\mathbf{p}_2$ as the probability vector, and normalized to sum to $1$, as described next. First, we generated a batch containing $500$ observations of $\mathbf{p}_1$ and $250$ observations of $\mathbf{p}_2$, each with $1000$ multinomial trials. Second, we added a batch containing $250$ observations of $\mathbf{p}_2$ only, each with $10^4$ multinomial trials. To summarize, the total number of observations is $n=1000$, given explicitly by
\begin{equation}
    \widetilde{\mathbf{x}}_i = \frac{\hat{\mathbf{x}}_i}{\sum_{j=1}^m \hat{\mathbf{x}}_{i}[j]}, \qquad 
    \hat{\mathbf{x}}_i \sim 
    \begin{dcases}
    \operatorname{Multinomial}(10^3,\mathbf{p}_1), & 1\leq i \leq 500, \\
    \operatorname{Multinomial}(10^3,\mathbf{p}_2), & 501 \leq i \leq 750, \\
    \operatorname{Multinomial}(10^4,\mathbf{p}_2), & 751 \leq i \leq 1000.
    \end{dcases} \label{eq:scRNAseq data gen 2}
\end{equation}
Therefore, the dataset consists of $500$ (normalized) multinomial observations of $\mathbf{p}_1$, followed by $500$ (normalized) multinomial observations of $\mathbf{p}_2$. While all observations of $\mathbf{p}_1$ are with $10^3$ multinomial trials, the observations of $\mathbf{p}_2$ are split between $250$ observations with $10^3$ multinomial trials, and $250$ observations with $10^4$ multinomial trials. Evidently, we can write
\begin{equation}
\widetilde{\mathbf{x}}_i = \mathbb{E} [\widetilde{\mathbf{x}}_i] + \eta_i = \mathbf{p}_{\ell_i} + \eta_i, \qquad \ell_i = \begin{dcases}
1, & 1\leq i \leq 500, \\
2, & 501 \leq i \leq 1000, 
\end{dcases} \label{eq:scRNAseq data gen 3}
\end{equation}
where $\eta_i$ is a zero-mean noise vector (arising from the multinomial sampling) satisfying that $\mathbb{E}\Vert \eta_i\Vert^2$ is significantly smaller (by a factor of $10$ roughly) for $751 \leq i \leq 1000$ compared to $1\leq i \leq 750$.

Using the noisy observations $\widetilde{\mathbf{x}}_1,\ldots,\widetilde{\mathbf{x}}_n$, we formed the noisy kernel matrix $\widetilde{K}$ of~\eqref{eq:K def} with $\varepsilon = 2\cdot 10^{-5}$, computed the corresponding matrix $\widetilde{W}^{(d)}$ using Algorithm~\ref{alg:SK sym} with $\delta=10^{-12}$, and evaluated the matrices $\widetilde{W}^{(r)}$, $\widetilde{W}^{(s)}$ according to~\eqref{eq:W_r def} and~\eqref{eq:W_s def}. Our methodology for choosing $\varepsilon$ was to take it to be the smallest possible such that Algorithm~\ref{alg:SK sym} converges within the desired tolerance (specifically, in this experiment we set a maximum of $10^6$ iterations for the algorithm). We note that if $\varepsilon$ is too small, then $\widetilde{K}$ becomes too sparse (approximately), and the doubly-stochastic normalization may become numerically ill-posed. 

Figure~\ref{fig:Noisy affinity matrices small epsilon} illustrates the values (in logarithmic scale) of the obtained affinity matrices $\widetilde{W}^{(d)}$, $\widetilde{W}^{(r)}$, $\widetilde{W}^{(s)}$. It is evident that the affinity matrix from the doubly-stochastic normalization accurately describes the relationships between the data points. That is, $\widetilde{W}^{(d)}$ indicates the similarities within the two groups of cell types (i.e., $\mathbf{p}_1$ and $\mathbf{p}_2$), but also the dissimilarities between them, regardless of batch association. On the other hand, the affinity matrices from the row-stochastic and the symmetric normalizations are not loyal to the grouping according to cell types, but rather to batch association. In particular, $\widetilde{W}^{(r)}$ and $\widetilde{W}^{(s)}$ highlight the observations from the second batch (observations $751$--$1000$) as being most similar to all other observations. Clearly, the fundamental issue here is the heteroskedasticity of the noise, and specifically, the fact that the noise in the last $250$ observations is considerably smaller than the noise in all the other observations. 

\begin{figure} 
  \centering
  	{
  	\subfloat[][Doubly-stochastic normalization~\eqref{eq:W_d def}]  
  	{
    \includegraphics[width=0.33\textwidth]{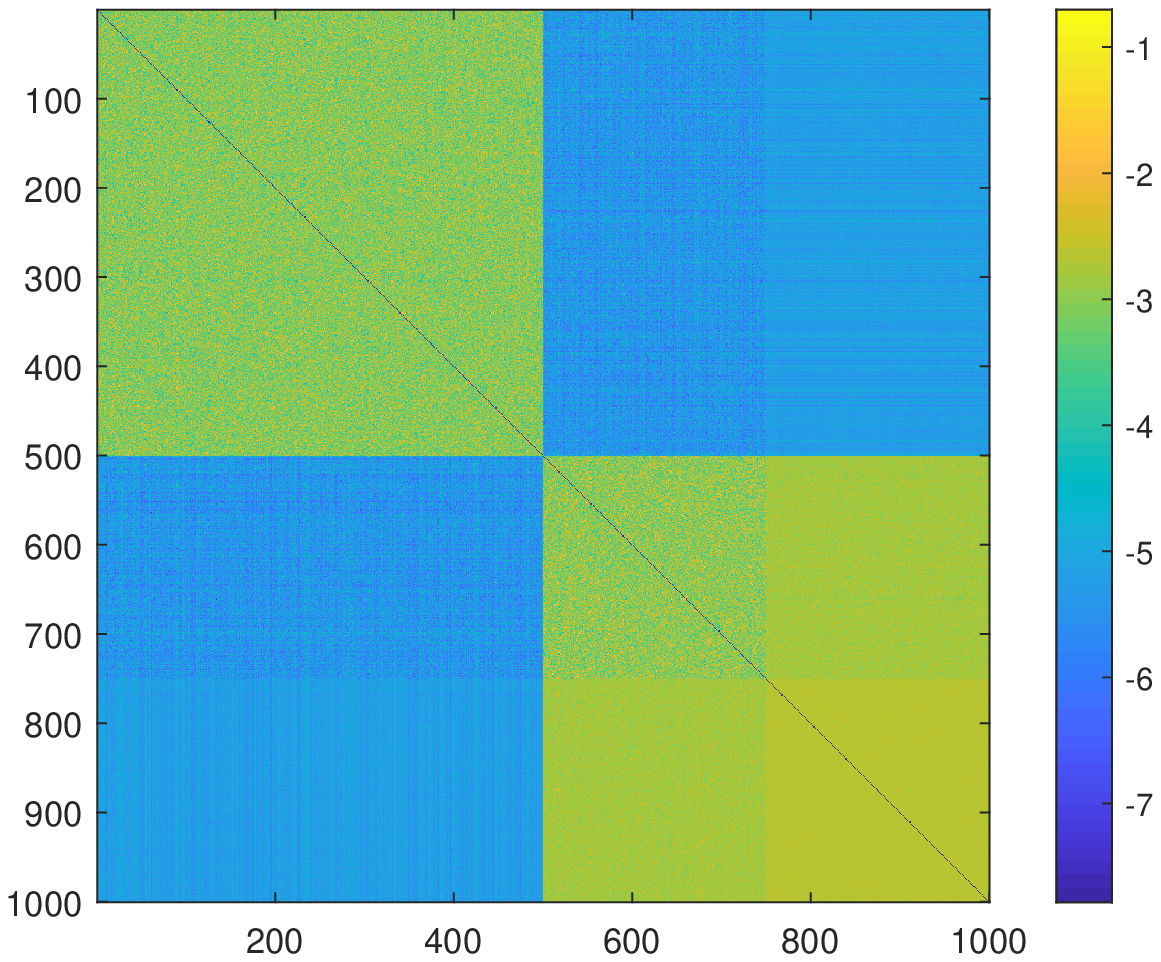} 
    }
    \subfloat[][Row-stochastic normalization~\eqref{eq:W_r def}] 
  	{
    \includegraphics[width=0.33\textwidth]{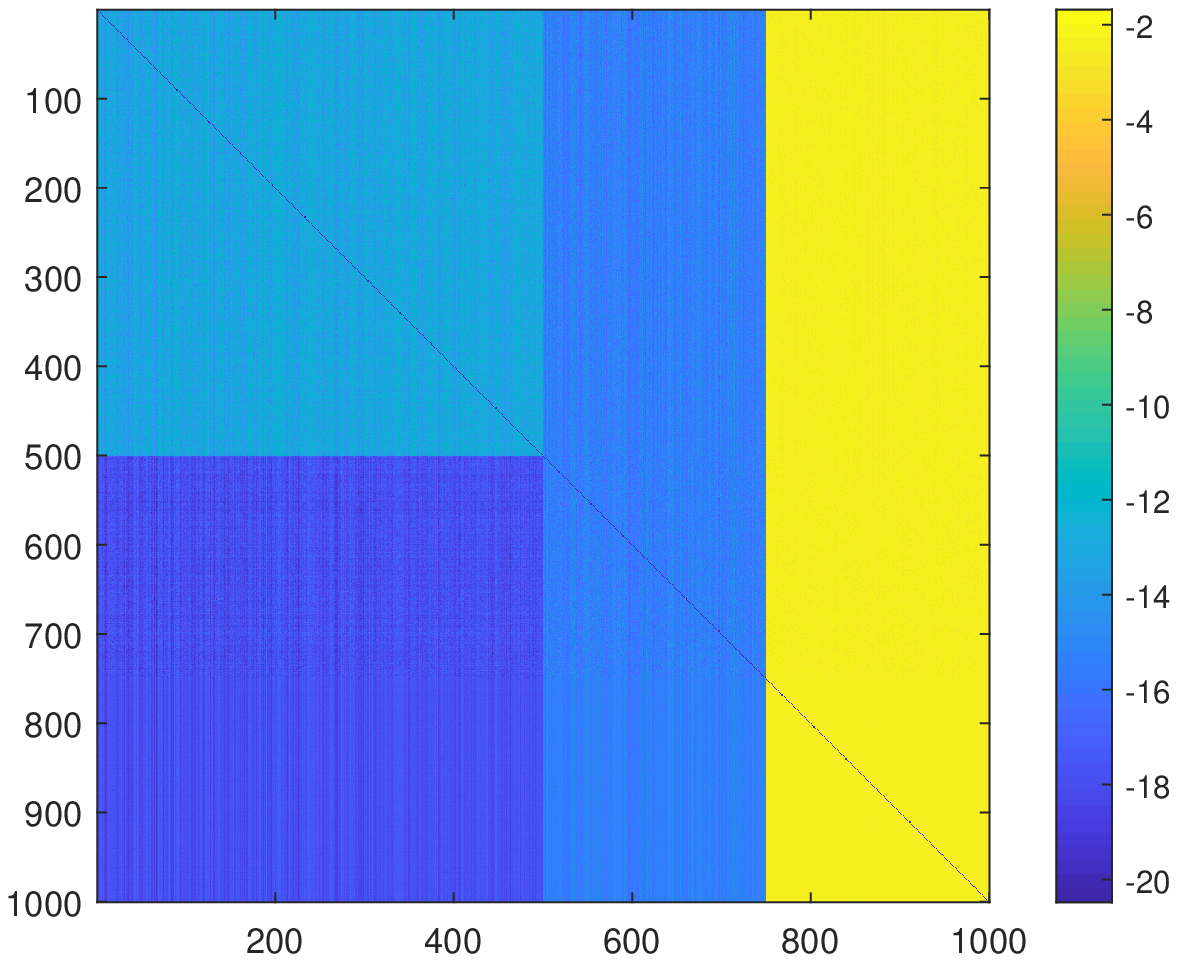} \label{fig:Noisy affinity matrices small epsilon, row-stochastic}
    }
    \subfloat[][Symmetric normalization~\eqref{eq:W_s def}]  
  	{
    \includegraphics[width=0.33\textwidth]{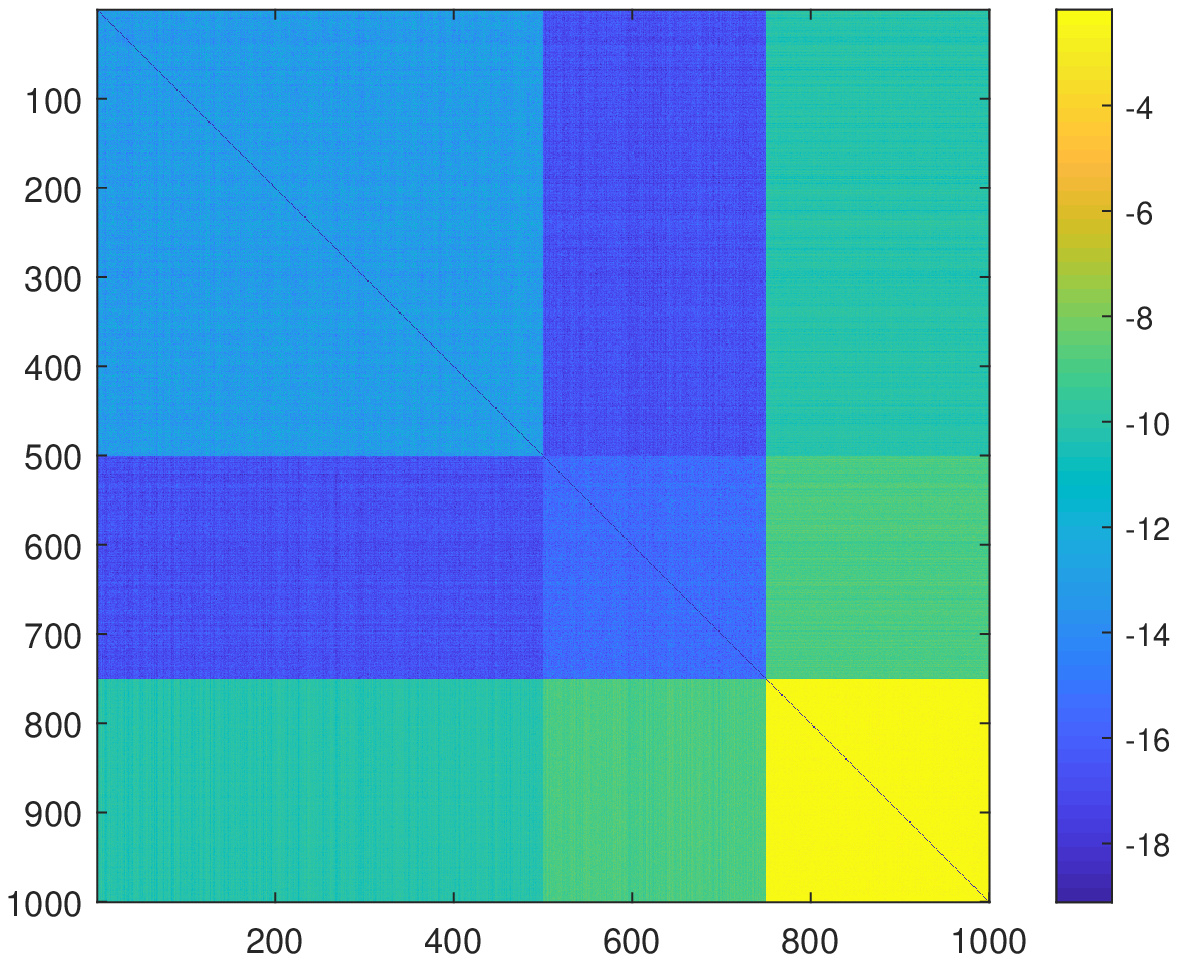}
    } 
    }
    \caption
    {Entries of the affinity matrices obtained from different normalizations in logarithmic scale (from left to right: $\log_{10}(\widetilde{W}^{(d)})$, $\log_{10}(\widetilde{W}^{(r)})$, $\log_{10}(\widetilde{W}^{(s)})$), for single-cell RNA sequence data simulated according to~\eqref{eq:scRNAseq data gen 1}--\eqref{eq:scRNAseq data gen 3}, with $n=1000$, $m=4000$.} \label{fig:Noisy affinity matrices small epsilon}
    \end{figure} 

One of the main goals of exploratory analysis of scRNA-seq data is to identify different cell types. Towards that end, non-linear dimensionality reduction techniques are often employed, among which \textit{t-distributed stochastic neighbor embedding} (t-SNE)~\cite{maaten2008visualizing} is perhaps the most prominent~\cite{linderman2019fast,tirosh2016dissecting,villani2017single,habib2016div}. For its operation, t-SNE employs an affinity matrix which is a close variant of the row-stochastic normalization~\eqref{eq:W_r def}, where the kernel width parameter $\varepsilon$ in~\eqref{eq:K def} is allowed to vary between different rows of $K$, and the resulting row-stochastic matrix is symmetrized by averaging it with its transpose. The different kernel widths are determined by a parameter called the \textit{perplexity}, which is related to the entropy of each row of the resulting affinity matrix. 

Even though the affinity matrix employed by t-SNE is a modification of the standard row-stochastic normalization, and uses a different value of $\varepsilon$ for each row, it is still expected to suffer from the inherent bias observed in Figure~\ref{fig:Noisy affinity matrices small epsilon, row-stochastic}. Specifically, note that the order of the entries in each row of $\widetilde{W}^{(r)}$ (when sorted by their values) does not depend on $\varepsilon$, and only on the noisy pair-wise distances $\Vert \widetilde{\mathbf{x}}_i - \widetilde{\mathbf{x}}_j \Vert^2$, which are strongly biased by the magnitudes of the noise, as evident from Figure~\ref{fig:Noisy affinity matrices small epsilon, row-stochastic}.

In Figures~\ref{fig:t-sne perp 10},\ref{fig:t-sne perp 30},\ref{fig:t-sne perp 100} we demonstrate the two-dimensional visualization obtained from t-SNE for the dataset $\widetilde{\mathbf{x}}_1,\ldots,\widetilde{\mathbf{x}}_n$, using typical perplexity values of $10,30,100$. We used MATLAB's standard implementation of t-SNE, activating the option of forcing the algorithm to be exact (i.e., without approximating the affinity matrix). All other parameters of t-SNE were set to their default values suggested by the code (we also mention that the default suggested perplexity is $30$). 

In Figure~\ref{fig:t-sne doubly-stochastic} we display the two-dimensional visualization obtained from t-SNE when replacing its default affinity matrix construction with the doubly-stochastic matrix $\widetilde{W}^{(d)}$ (obtained using $\varepsilon=2\cdot 10^{-5}$), while leaving all other aspects of t-SNE unchanged. Since the optimization procedure in t-SNE is affected by randomness, we ran the experiment several times to verify that the results we exhibit are consistent.

While there are only two types of cell in the data ($\mathbf{p}_1$ and $\mathbf{p}_2$), no clear evidence of this fact can be found in the visualizations by t-SNE (Figures~\ref{fig:t-sne perp 10},\ref{fig:t-sne perp 30},\ref{fig:t-sne perp 100}). Furthermore, the visualizations by t-SNE do not provide any noticeable separation between the cell types. On the other hand, the visualization obtained by modifying the t-SNE to employ the doubly-stochastic affinity matrix $\widetilde{W}^{(d)}$  (Figure~\ref{fig:t-sne doubly-stochastic}) allows one to easily identify and distinguish between the two cell types. 

\begin{figure} 
  \centering
  	{
  	\subfloat[][t-SNE, $\text{perplexity}=10$]
  	{
    \includegraphics[width=0.3\textwidth]{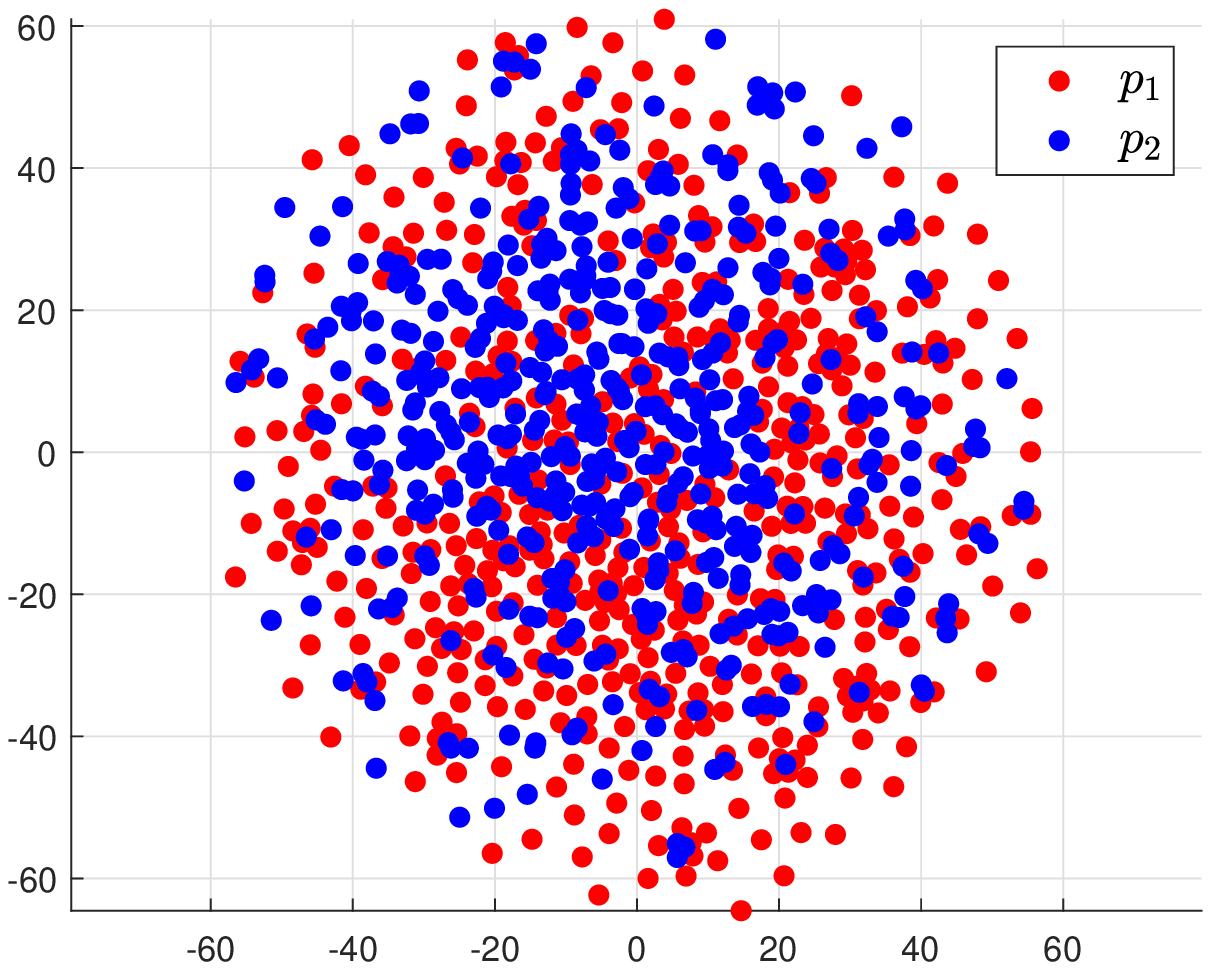} \label{fig:t-sne perp 10}
    }
    \subfloat[][t-SNE, $\text{perplexity}=30$] 
  	{
    \includegraphics[width=0.3\textwidth]{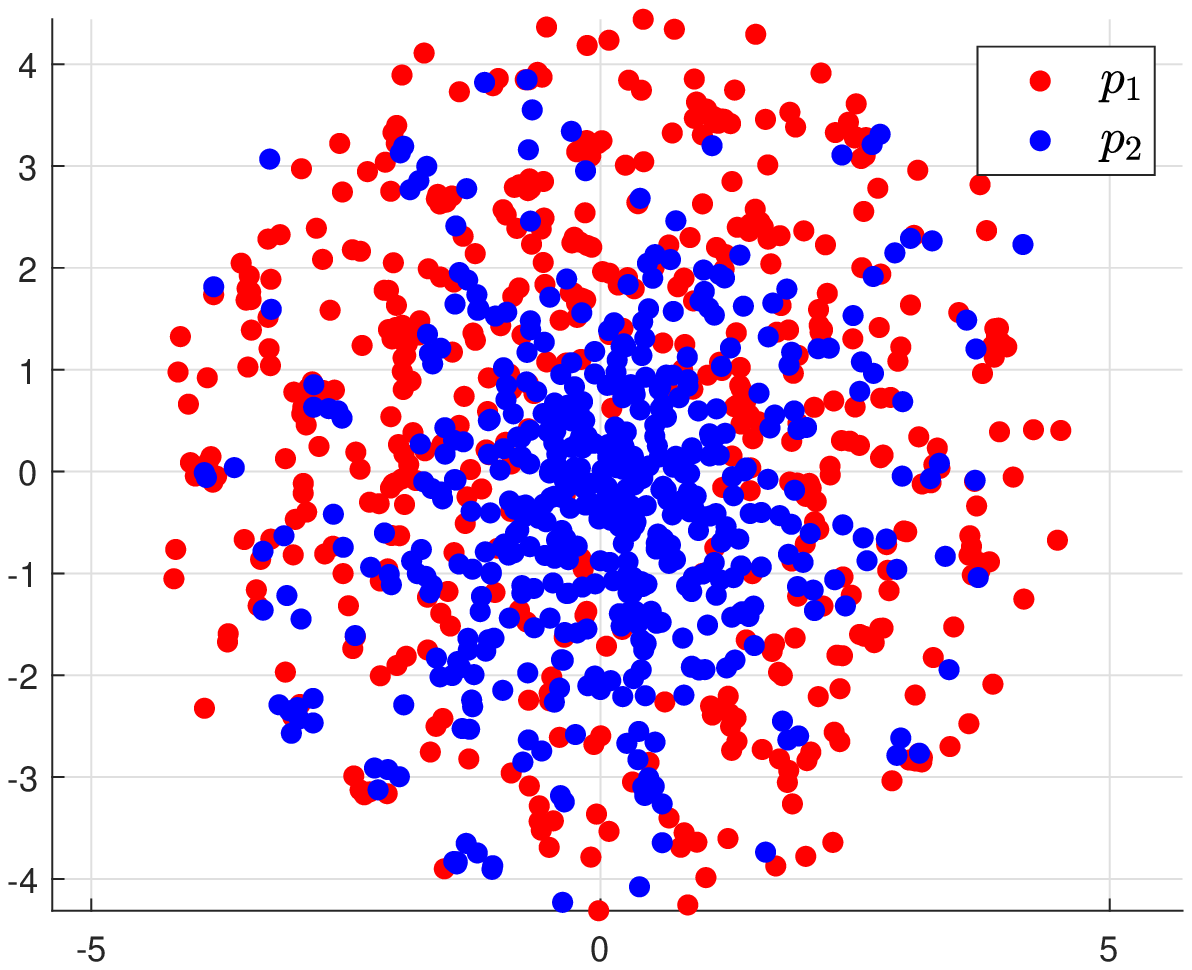} \label{fig:t-sne perp 30}
    }
    \\
    \subfloat[][t-SNE, $\text{perplexity}=100$]  
  	{
    \includegraphics[width=0.3\textwidth]{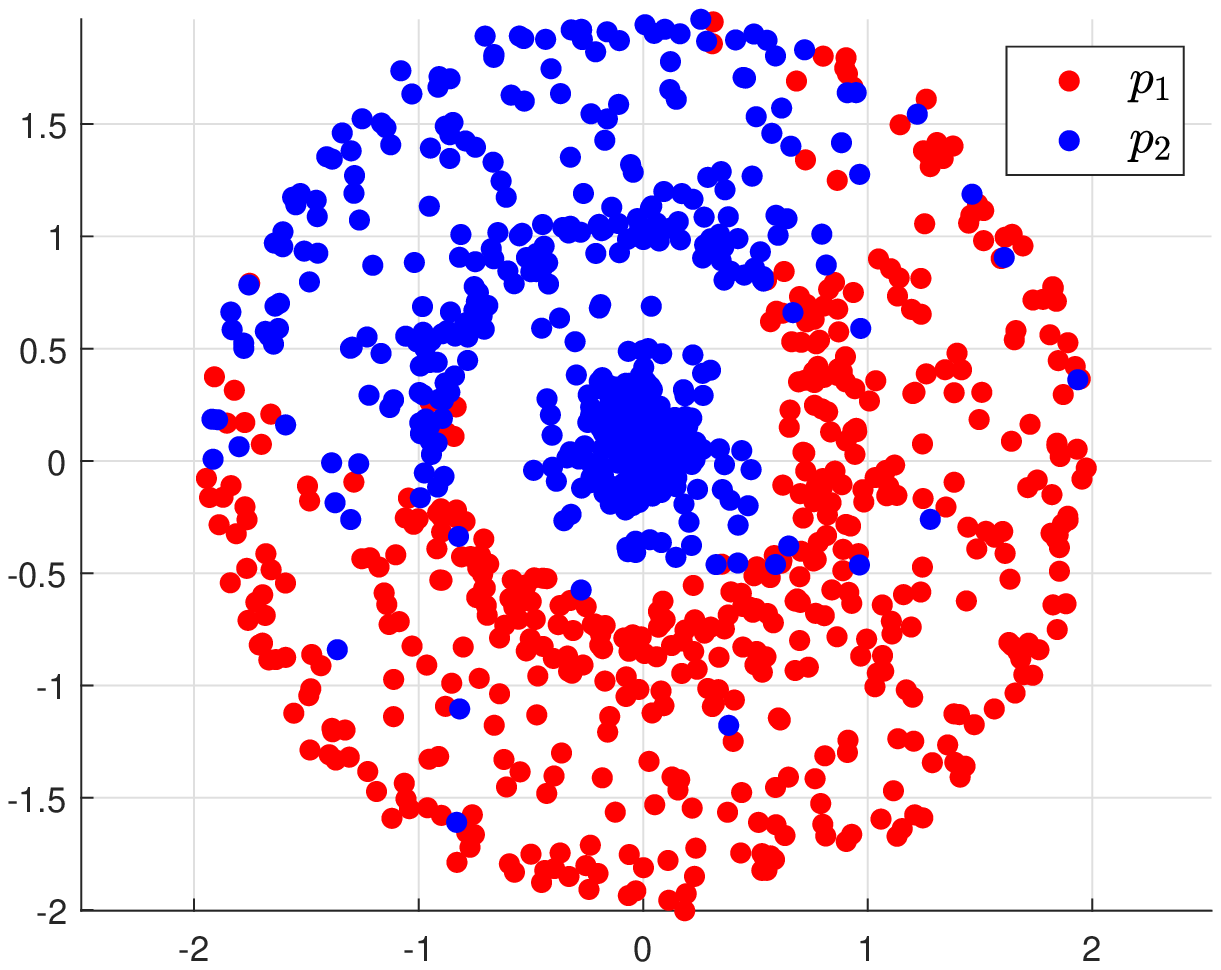}  \label{fig:t-sne perp 100}
    }
    \subfloat[][t-SNE with doubly-stochastic affinity $\widetilde{W}^{(d)}$] 
  	{
    \includegraphics[width=0.3\textwidth]{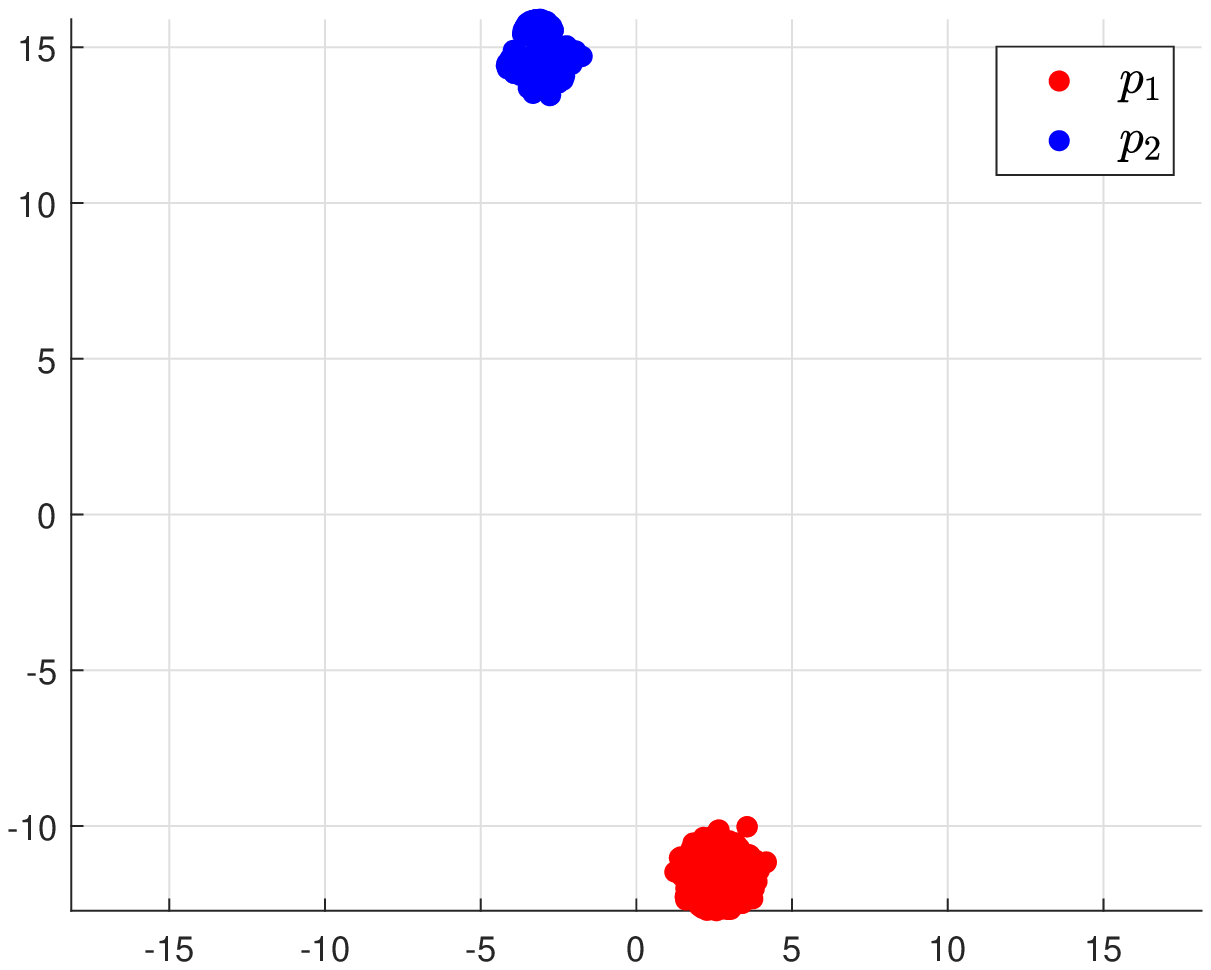} \label{fig:t-sne doubly-stochastic}
    } 

    }
    \caption
    {Two-dimensional visualization from t-SNE with different perplexity values (Figures~\ref{fig:t-sne perp 10},\ref{fig:t-sne perp 30},\ref{fig:t-sne perp 100}), and from t-SNE modified to use the doubly-stochastic affinity matrix $\widetilde{W}^{(d)}$ (Figure~\ref{fig:t-sne doubly-stochastic}). The dataset is a simulated single-cell RNA sequence data (see~\eqref{eq:scRNAseq data gen 1}--\eqref{eq:scRNAseq data gen 3}) with $n=1000$, $m=4000$.} \label{fig:t-SNE single-cell RNA seq}
    \end{figure} 
    
\subsubsection{Experimental data}
In our second example for scRNA-seq, we use an experimental dataset of purified Peripheral Blood Mononuclear Cells (PBMC)~\cite{zheng2017massively}, which includes $94654$ cells and $32733$ genes. This dataset is particularly advantageous for our purposes since each cell in the experiment was labeled according to a known cell type (with $10$ different types in total). While this particular dataset does not include different experimental batches (as in the previous simulated example), there is nonetheless inherent variability in the read counts associated with different cell types. To demonstrate the advantage of the doubly-stochastic normalization over the row-stochastic or symmetric normalizations, we focus on the two cell types in the data that have the largest difference in their read counts (on average), which are the CD14 and CD34 cells. Specifically, the CD34 cells have roughly four times more read counts on average than the CD14 cells. 

We randomly sampled $n_1= 10^3$ cells out of all CD14 cells, sampled $n_2= 10^3$ cells out of all CD34 cells, and concatenated their gene expressions (using all genes) into a matrix of size $32733 \times (n_1+n_2)$. We then normalized each column of this matrix to sum to $1$ (which is a standard procedure used in scRNA-seq for normalizing the read count of each cell, see also Section~\ref{subsubsec:scRNA-seq simulated data}), and denoted the resulting columns by $\widetilde{\mathbf{x}}_1,\ldots,\widetilde{\mathbf{x}}_{n_1+n_2}$. That is, $\widetilde{\mathbf{x}}_1,\ldots,\widetilde{\mathbf{x}}_{n_1}$ are the normalized gene expressions of the sampled CD14 cells, and $\widetilde{\mathbf{x}}_{n_1+1},\ldots,\widetilde{\mathbf{x}}_{n_1+n_2}$ are the normalized gene expressions of the sampled CD34 cells.  We then formed the kernel matrix $\widetilde{K}$ of~\eqref{eq:K def} with $\varepsilon = 5\cdot 10^{-4}$, computed the corresponding matrix $\widetilde{W}^{(d)}$ using Algorithm~\ref{alg:SK sym} with $\delta=10^{-12}$, and evaluated the matrices $\widetilde{W}^{(r)}$, $\widetilde{W}^{(s)}$ according to~\eqref{eq:W_r def} and~\eqref{eq:W_s def}. We mention that other values of $\varepsilon$ produce similar results to what we report next.

Figure~\ref{fig:scRNA-seq real data affinities} illustrates the values (in logarithmic scale) of the obtained affinity matrices $\widetilde{W}^{(d)}$, $\widetilde{W}^{(r)}$, $\widetilde{W}^{(s)}$. It is evident that the affinity matrices form the doubly-stochastic and the symmetric normalizations accurately reflect the structure of the data, as they assign large affinities between cells of the same type and small affinities between cells of different type. The row-stochastic normalization, on the other hand, assigns large affinities between CD14 cells and CD34 cells, which is clearly a bias from the fact that the CD34 cells are less noisy compared to the CD14 cells (due to the difference between their read counts). 
\begin{figure} 
  \centering
  	{
  	\subfloat[][Doubly-stochastic normalization~\eqref{eq:W_d def}]  
  	{
    \includegraphics[width=0.33\textwidth]{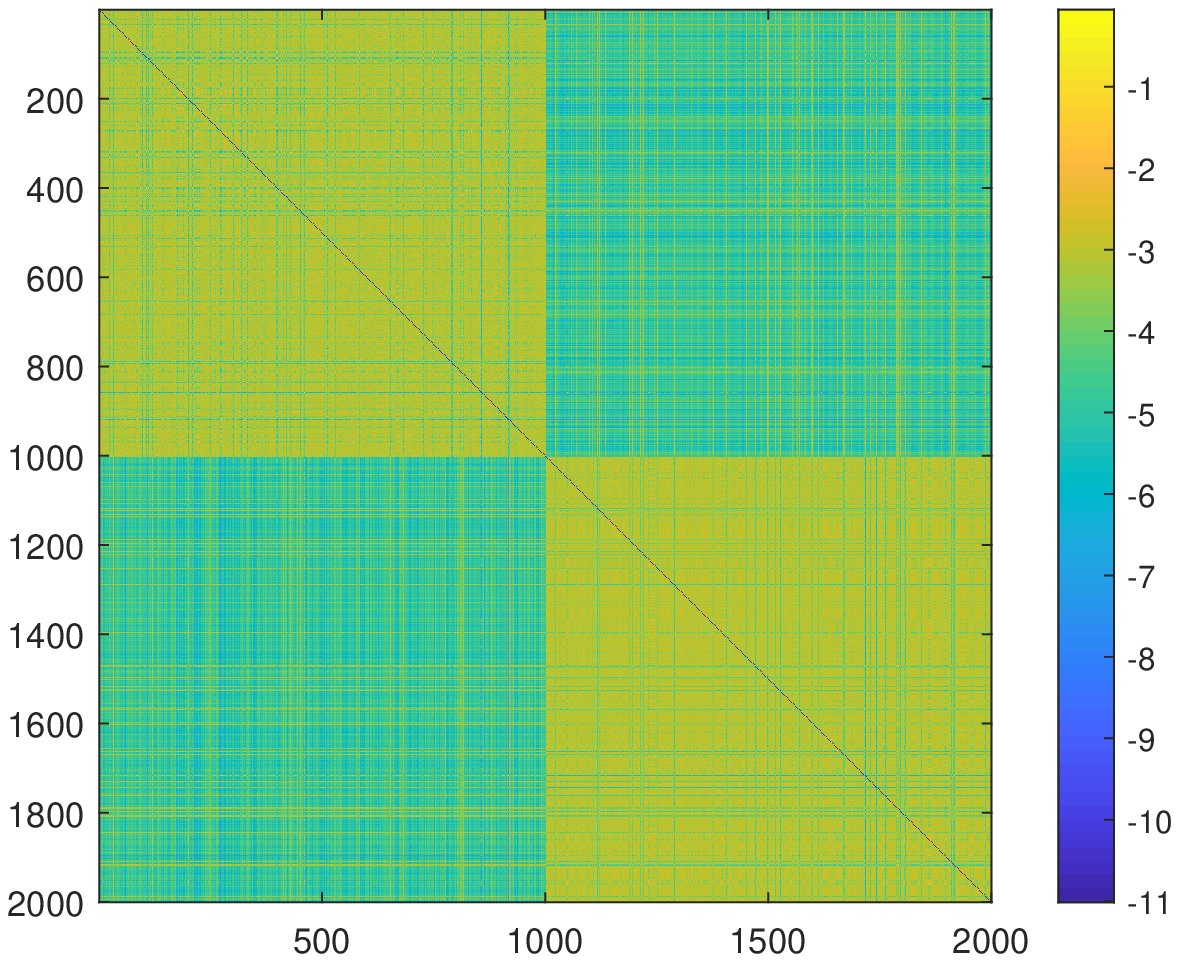} 
    }
    \subfloat[][Row-stochastic normalization~\eqref{eq:W_r def}] 
  	{
    \includegraphics[width=0.33\textwidth]{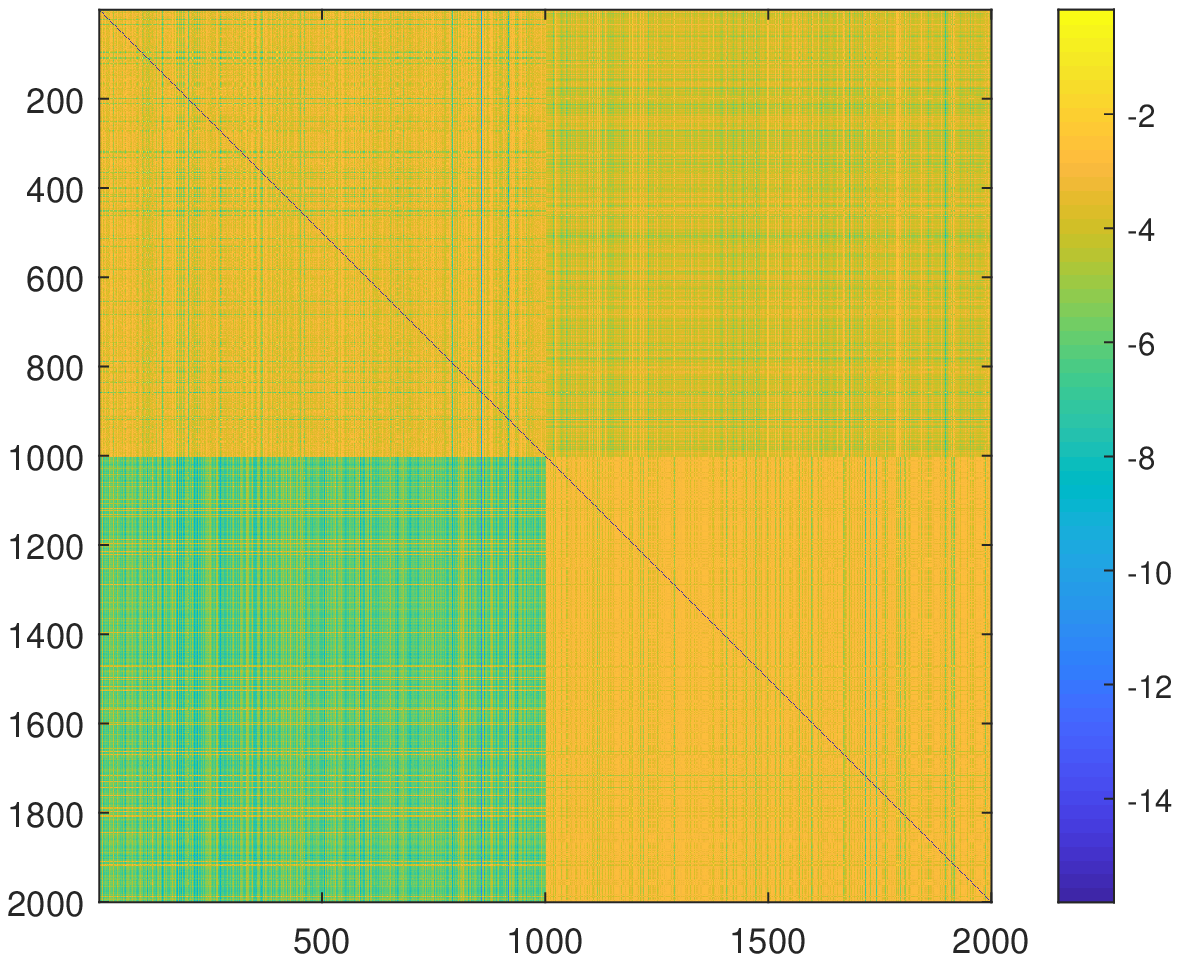} \label{fig:Noisy affinity matrices small epsilon, row-stochastic}
    }
    \subfloat[][Symmetric normalization~\eqref{eq:W_s def}]  
  	{
    \includegraphics[width=0.33\textwidth]{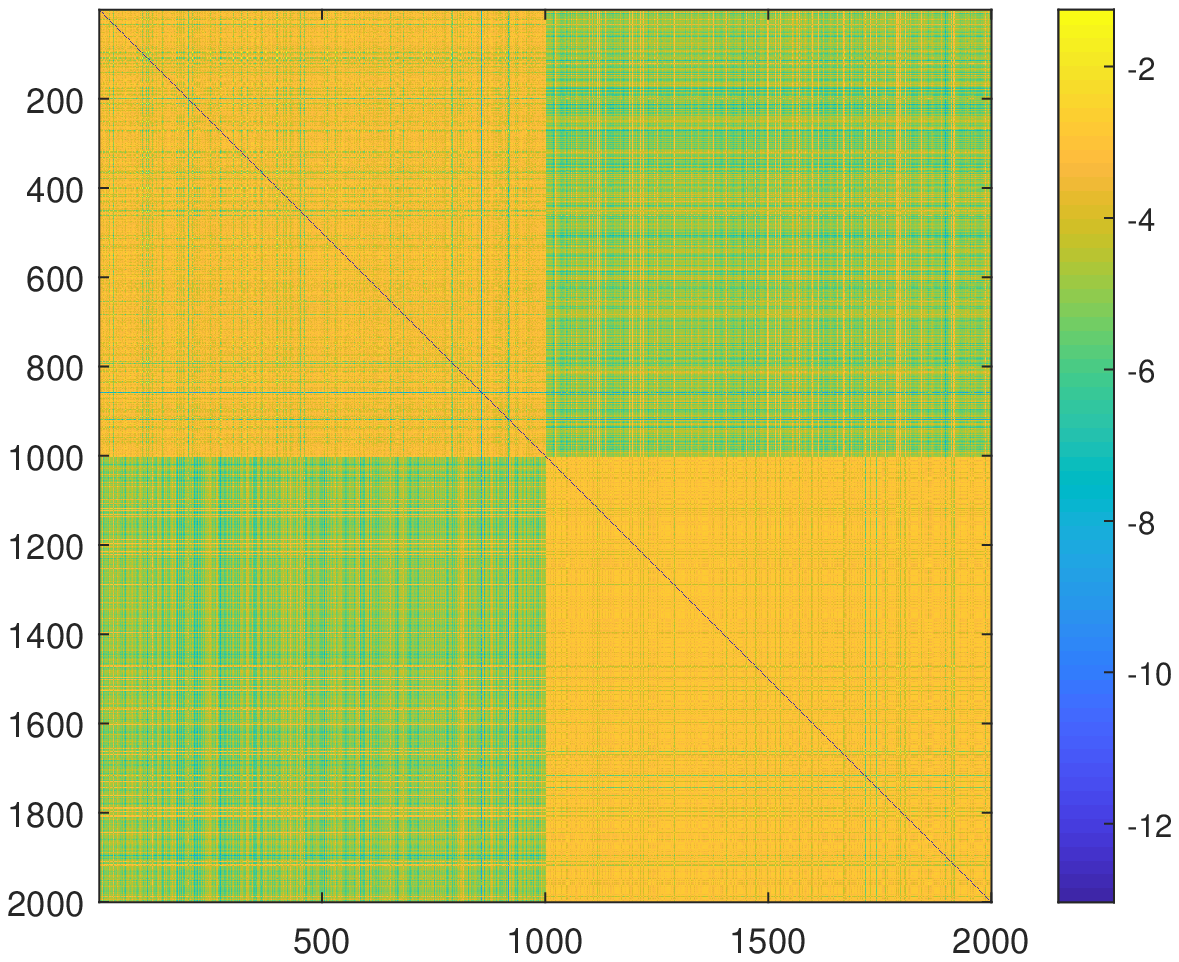}
    } 
    }
    \caption
    {Entries of the affinity matrices obtained from different normalizations in logarithmic scale (from left to right: $\log_{10}(\widetilde{W}^{(d)})$, $\log_{10}(\widetilde{W}^{(r)})$, $\log_{10}(\widetilde{W}^{(s)})$), computed from $1000$ CD14 cells and $1000$ CD34 cells from the purified PBMC dataset~\cite{zheng2017massively}.} \label{fig:scRNA-seq real data affinities}
    \end{figure} 

Aside from the qualitative differences between the affinity matrices depicted in Figure~\ref{fig:scRNA-seq real data affinities}, it is of interest to quantitatively assess their accuracy. Even though we do not have access to the corresponding clean affinity matrices ${W}^{(d)}$, ${W}^{(r)}$, ${W}^{(s)}$, we can make use of the logical reasoning that the nearest neighbors of any given reference cell, defined as the cells with largest affinities to that reference cell, should belong to the same cell type as the reference cell. Following this logic, for each cell $i$ we first found its $k$ nearest neighbors, which are given by the $k$ indices with largest entries in the $i$'th row of a given affinity matrix ($\widetilde{W}^{(d)}$, $\widetilde{W}^{(r)}$, or $\widetilde{W}^{(s)}$). Then, as a measure of error, for each cell $i$ we found the proportion of its $k$ nearest neighbors that do not share its cell type. We averaged this proportion for all cells $i=1,\ldots,n_1+n_2$, and furthermore averaged these results over $20$ randomized trials of sampling from the CD14 and CD34 cells.

Figure~\ref{fig:scRNA-seq real data error in nearest neighbors} depicts the resulting proportions of inconsistent cell types as a function of the number of nearest neighbors $k$ for each of the affinity matrices $\widetilde{W}^{(d)}$, $\widetilde{W}^{(r)}$, and $\widetilde{W}^{(s)}$. It is evident that according to our measure of cell type consistency, the affinity matrix from the doubly-stochastic normalization provides the lowest proportion of incorrectly determined near neighbors for all choices of $k$, establishing the advantage of the doubly-stochastic normalization. In particular, the proportion of the nearest neighbor (i.e., $k=1$) with inconsistent cell type using the doubly-stochastic normalization is about $5$ times less than that of the symmetric normalization, and about $20$ times less than that of the row-stochastic normalization. 

\begin{figure} 
  \centering
  	{
    \includegraphics[width=0.5\textwidth]{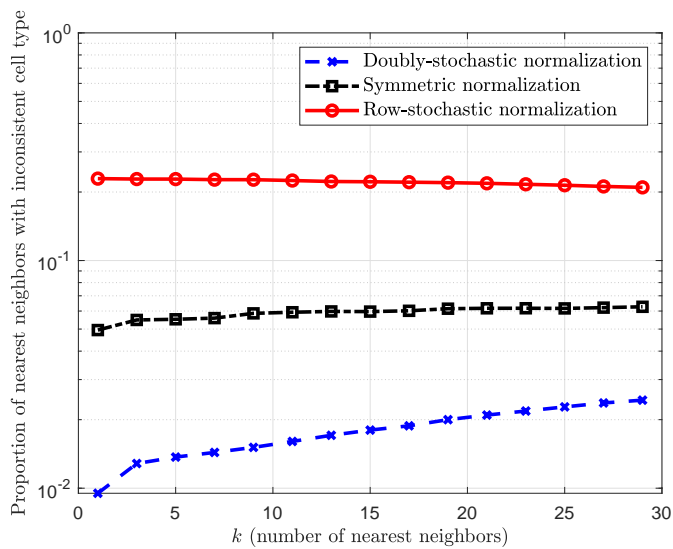} 
    }
    \caption
    {Inconsistency between the type of a cell and the types of its nearest neighbours, according to each of the affinity matrices $\widetilde{W}^{(d)}$, $\widetilde{W}^{(r)}$, $\widetilde{W}^{(s)}$, and using $1000$ CD14 cells and $1000$ CD34 cells from the purified PBMC dataset~\cite{zheng2017massively}. The $x$-axis is the number of nearest neighbors chosen for each cell (according to the largest entries in each row of a given affinity matrix), and the $y$-axis is the proportion of these nearest neighbors with a cell type that is different from the type of the reference cell (whose nearest neighbors are considered), averaged over all cells in the dataset and over $20$ randomized trials (of sampling from the CD14 cells and the CD34 cells). } \label{fig:scRNA-seq real data error in nearest neighbors}
    \end{figure} 

\subsubsection{Validity of the asymptotic model for scRNA-seq} \label{subsubsec:scRNA-seq model asymptotic discussion}
We conclude our scRNA-seq example with a brief discussion on the validity of Theorem~\ref{thm:Main theorem} in an asymptotic setting of scRNA-seq, where the number of cells $n$ and the kernel width $\varepsilon$ are fixed, and the number of genes $m$ is increasing. Suppose that the gene expression levels for each cell are sampled from a multinomial variable~\cite{townes2019feature}, where the number of multinational trials may differ between cells. That is, let $\hat{\mathbf{x}}_{i}\sim\operatorname{Multinomial(r_i,\mathbf{p}_i)}$, where $r_i$ is the read count for the $i$'th cell, $\mathbf{p}_{i}[j]$ is the underlying proportion for the expression level of the $j$'th gene in the $i$'th cell, and $\sum_{j=1}^m \mathbf{p}_{i}[j] = 1$ for all $i$. When preprocessing scRNA-seq data, an important first step is to normalize the gene expression levels by the number of read counts. Hence, we define $\widetilde{\mathbf{x}}_{i} = \hat{\mathbf{x}}_{i}/r_i$, so that $\mathbb{E}[\widetilde{\mathbf{x}}_{i}] = \mathbf{x}_{i} =  \mathbf{p}_{i}$. In this case, the matrix $\Sigma_i = \mathbb{E}[\eta_i \eta_i^T]$ is equal to the covariance matrix of the multinomial $\hat{\mathbf{x}}_i$ divided by $r_i^2$. Using Theorem 1 in~\cite{benasseni2012new} (which provides an inequality on the eigenvalues of the covariance of a multinomial),
\begin{equation}
    \Vert \Sigma_i \Vert_2 \leq \max_{j=1,\ldots,m} \mathbf{p}_{i}[j]/r_i \leq 1/r_i.
\end{equation}
To consider an asymptotic setting, we think of an experimental setup where an increasing number of genes is sequenced. It is clear that in this experimental setup the number of read counts $r_i$ for each cell needs to be controlled appropriately as a function of $m$. Since $\Vert \mathbf{x}_i \Vert^2_2 = \sum_{j=1}^m (\mathbf{p}_{i}[j])^2 \leq \sum_{j=1}^m \mathbf{p}_{i}[j] = 1$, it is evident that the conditions in Theorem~\ref{thm:Main theorem} hold if $r_i = \Omega(m)$, a condition which was recently identified as important for large-scale scRNA-seq experiments~\cite{zhang2020determining}.

\section{Summary and discussion}
In this work, we investigated the robustness of the doubly-stochastic normalization to heteroskedastic noise, both from a theoretical perspective and from a numerical one. Our results imply that the doubly-stochastic normalization is advantageous over the popular row-stochastic and symmetric normalizations, particularly when the data at hand is high-dimensional and suffers from inherent heteroskedasticity. Moreover, our experiments suggest that incorporating the doubly-stochastic normalization into various data analysis, visualization, and processing techniques for real-world datasets can be worthwhile. The doubly-stochastic normalization is particularly appealing due to is simplicity, solid theoretical foundation, and resemblance to the row-stochastic/symmetric normalizations (which proved useful in countless applications).

The results reported in this work naturally give rise to several possible future research directions. On the theoretical side, it is of interest to characterize the convergence rate of $\widetilde{W}^{(d)}$ to $W^{(d)}$ also in terms of the number of points $n$ and the covariance matrices $\{\Sigma_i\}$ explicitly. As a particular simpler case, one may consider the high-dimensional setting where both $n$ and $m$ tend to infinity, while the quantity $n/m$ is fixed (or tends to a fixed constant). 
On the practical side, it is of interest to investigate how to best incorporate the affinity matrix from the doubly-stochastic normalization into data analysis and visualization techniques. To that end, it is desirable to derive a method for picking the kernel parameter $\varepsilon$ automatically, or in more generality, to determine how to make use of a variable kernel width (similarly to~\cite{zelnik2005self}) while retaining the robustness to heteroskedastic noise.

\section{Acknowledgements}
We would like to thank Boaz Nadler for his useful comments and suggestions. 
B.L, R.R.C, and Y.K. acknowledge support by NIH grant R01GM131642. R.R.C and Y.K acknowledge support by NIH grant R01HG008383. Y.K. acknowledges support by NIH grant 2P50CA121974. R.R.C acknowledges support by NIH grant 5R01NS10004903.

\begin{appendices}

\section{Proof of Lemma~\ref{lem:existence and uniquness}} \label{appendix:existence and uniqueness proof}
We first recall the definition of a \textit{fully indecomposable} matrix~\cite{bapat1997nonnegative}. A matrix $B$ is called fully indecomposable if there are \textbf{no} permutation matrices $P$ and $Q$ such that 
\begin{equation}
    P B Q = 
    \begin{bmatrix}
    B_1 & \mathbf{0} \\
    B_2 & B_3
    \end{bmatrix}, \label{eq:fully indecomposable form}
\end{equation}
with $B_1$ square. We now proceed to show that $A$ from Lemma~\ref{lem:existence and uniquness} is fully indecomposable.
Since the only zeros in $A$ are on its main diagonal, there is only one zero in every row and every column of $A$. Consequently, any permutation of the rows and columns of $A$ would retain this property, namely have a single zero in every row and every column. Therefore, if $n> 2$, it is impossible to find $P$ and $Q$ such that~\eqref{eq:fully indecomposable form} would hold for $B=A$, since there cannot be a block of zeros in $P A Q$ whose number of rows or columns is greater than $1$. Hence, $A$ is fully indecomposable, and the existence and uniqueness of $\mathbf{d} = [d_1,\ldots,d_n] > 0$ follows from Lemma 4.1 in~\cite{knight2008sinkhorn}.

\section{Proof of Theorem~\ref{thm:Main theorem}} \label{sec:proof of main theorem}
Throughout this proof we omit the superscript ${(m)}$ from the quantities $\mathbf{x}_i^{(m)}$, $\widetilde{\mathbf{x}}_i^{(m)}$, $\eta_i^{(m)}$, $\Sigma_{i}^{(m)}$, $K^{(m)}$, $\widetilde{K}^{(m)}$, ${W}^{(d),(m)}$, $\widetilde{W}^{(d),(m)}$, and it should be noted that the resulting notation corresponds to sequences in the dimension $m$ where $n$ and $\varepsilon$ are fixed. 

Let us define
\begin{equation}
    {u}_i = {d}_i \operatorname{exp}({-\Vert \mathbf{x}_i \Vert^2/\varepsilon}), \qquad {H}_{i,j} = \begin{dcases}
    \operatorname{exp}({2\langle \mathbf{x}_i,\mathbf{x}_j \rangle/\varepsilon}), & i\neq j,\\
    0, & i=j, \label{eq:u_tilde and H_tilde def}
    \end{dcases}
    \end{equation}
for $i,j=1,\ldots,n$. By the definition of ${W}^{(d)}$ in~\eqref{eq:W_d def}, for $i\neq j$ we can write
\begin{equation}
    {W}^{(d)}_{i,j} = {d}_i \operatorname{exp}(-\Vert \mathbf{x}_i - \mathbf{x}_j \Vert^2/\varepsilon) {d}_j = {d}_i e^{-\Vert \mathbf{x}_i \Vert^2/\varepsilon} e^{2\langle \mathbf{x}_i,\mathbf{x}_j \rangle/\varepsilon} e^{-\Vert \mathbf{x}_j\Vert ^2/\varepsilon} {d}_j = {u}_i {H}_{i,j} {u}_j.
\end{equation}
Analogously, we define $\widetilde{H}_{i,j}$ and $\widetilde{u}_i$ by replacing $\{\mathbf{x}_i\}$ and $\{d_i\}$ in~\eqref{eq:u_tilde and H_tilde def} with $\{\widetilde{\mathbf{x}}_i\}$ and $\{\widetilde{d}_i\}$, respectively, and we have that $\widetilde{W}_{i,j} = \widetilde{u}_i \widetilde{H}_{i,j} \widetilde{u}_j$.

Let $\odot$ denote the Hadamard (element-wise) product, $\mathbf{u}=[u_1,\ldots,u_n]^T$, and  $\widetilde{\mathbf{u}}=[\widetilde{u}_1,\ldots,\widetilde{u}_n]^T$. We can write
\begin{align}
    \Vert \widetilde{W}^{(d)} - W^{(d)} \Vert_F &= \Vert \operatorname{diag}(\widetilde{\mathbf{u}}) \widetilde{H} \operatorname{diag}(\widetilde{\mathbf{u}}) - \operatorname{diag}({\mathbf{u}}) H \operatorname{diag}({\mathbf{u}})\Vert_F \nonumber \\
    &= \Vert ({\mathbf{u}} {\mathbf{u}}^T)\odot(\widetilde{H} - H) + \widetilde{H} \odot (\widetilde{\mathbf{u}} \widetilde{\mathbf{u}}^T - \mathbf{u}\mathbf{u}^T)\Vert_F \nonumber \\ 
    &\leq \max_{i,j} \{ {u}_i {u}_j\} \cdot \Vert \widetilde{H} - H \Vert_F + \max_{i,j} \{\widetilde{H}_{i,j}\} \cdot \Vert \widetilde{\mathbf{u}} \widetilde{\mathbf{u}}^T - \mathbf{u}\mathbf{u}^T \Vert_F. \label{eq:W_tilde - W frob bound}
\end{align}
We begin by bounding the quantity $\Vert \widetilde{H} - H \Vert_F$, which is the subject of the following Lemma.
\begin{lem} \label{lem:H_tilde error}
For all $i\neq j$,
\begin{equation}
    \vert \widetilde{H}_{i,j} - H_{i,j} \vert = \mathcal{O}_p(m^{-1/2}).
\end{equation}
\end{lem}
\begin{proof}
Let us write
\begin{equation}
    \langle \widetilde{\mathbf{x}}_i , \widetilde{\mathbf{x}}_j \rangle = \langle \mathbf{x}_i , \mathbf{x}_j \rangle + \langle \mathbf{x}_i , \eta_j \rangle + \langle \eta_i , \mathbf{x}_j \rangle + \langle \eta_i , \eta_j \rangle. \label{eq:x_i_tilde x_j_tilde inner product}
\end{equation}
According to~\eqref{eq:noise def} and the conditions in Theorem~\ref{thm:Main theorem}, for $i\neq j$ we have
\begin{align}
    &\mathbb{E}\left\{ \langle \mathbf{x}_i , \eta_j \rangle + \langle \eta_i , \mathbf{x}_j \rangle + \langle \eta_i , \eta_j \rangle \right\} = 0, \\ 
    & \operatorname{Var}\left\{ \langle \mathbf{x}_i , \eta_j \rangle \right\} = \mathbb{E} [\mathbf{x}_i^T \eta_j \eta_j^T \mathbf{x}_i ] = \mathbf{x}_i^T \Sigma_j \mathbf{x}_i \leq \Vert \mathbf{x}_i \Vert^2 \Vert \Sigma_j \Vert_2 \leq C_\eta m^{-1}, \\ 
    & \operatorname{Var}\left\{ \langle \mathbf{x}_j , \eta_i \rangle \right\} = \mathbb{E} [\mathbf{x}_j^T \eta_i \eta_i^T \mathbf{x}_j ] = \mathbf{x}_j^T \Sigma_i \mathbf{x}_j \leq \Vert \mathbf{x}_j \Vert^2 \Vert \Sigma_i \Vert_2 \leq C_\eta m^{-1}, \\ 
    &\operatorname{Var}\left\{ \langle \eta_i , \eta_j \rangle \right\} = \mathbb{E}[\eta_i^T \eta_j \eta_j^T \eta_i] = \sum_{k=1}^m \sum_{\ell=1}^n \mathbb{E} [\eta_{i}[k] \eta_{i}[\ell]] ] \mathbb{E} \left[ \eta_{j}[k] \eta_{j}[\ell]\right] \nonumber \\ 
    &= \operatorname{Tr}\{ \Sigma_i \Sigma_j\} \leq m \Vert \Sigma_i \Sigma_j \Vert_2 \leq m \Vert \Sigma_i \Vert_2 \Vert \Sigma_j \Vert_2 \leq C_\eta^2 m^{-1}. 
\end{align}
Therefore, 
\begin{equation}
    \operatorname{Var}\left\{ \langle \mathbf{x}_i , \eta_j \rangle + \langle \eta_i , \mathbf{x}_j \rangle + \langle \eta_i , \eta_j \rangle \right\} \leq {\frac{6C_\eta+ 3C_\eta^2}{m}},
\end{equation}
where we used the inequality $(a+b+c)^2\leq 3(a^2+b^2+c^2)$.
Consequently, Chebyshev's inequality yields that for any $p>0$
\begin{equation}
    \operatorname{Pr}\left\{\left\vert \langle \mathbf{x}_i , \eta_j \rangle + \langle \eta_i , \mathbf{x}_j \rangle + \langle \eta_i , \eta_j \rangle \right\vert > \sqrt{\frac{6C_\eta + 3C_\eta^2}{m(1-p)}} \right\} \leq p,
\end{equation}
which implies
\begin{equation}
   \left\vert \langle \mathbf{x}_i , \eta_j \rangle + \langle \eta_i , \mathbf{x}_j \rangle + \langle \eta_i , \eta_j \rangle \right\vert = \mathcal{O}_p\left(m^{-1/2}\right).
\end{equation}
Using the above for $i\neq j$, a first-order Taylor expansion of $\operatorname{exp}(y)$ around $y=0$ gives
\begin{equation}
    \operatorname{exp}\{ 2(\langle \mathbf{x}_i , \eta_j \rangle + \langle \eta_i , \mathbf{x}_j \rangle + \langle \eta_i , \eta_j \rangle)/\varepsilon\} = 1 + \mathcal{O}_p(m^{-1/2}),
\end{equation}
and by~\eqref{eq:x_i_tilde x_j_tilde inner product} we have
\begin{equation}
    \widetilde{H}_{i,j} = e^{2\langle \widetilde{\mathbf{x}}_i,\widetilde{\mathbf{x}}_j \rangle/\varepsilon} = e^{2\langle \mathbf{x}_i,\mathbf{x}_j \rangle/\varepsilon} (1 + \mathcal{O}_p(m^{-1/2})) = H_{i,j} (1 + \mathcal{O}_p(m^{-1/2})) = H_{i,j} + \mathcal{O}_p(m^{-1/2}),
\end{equation}
where we used $H_{i,j} = e^{\langle \mathbf{x}_i,\mathbf{x}_j \rangle/\varepsilon} \leq e^{\Vert \mathbf{x}_i\Vert \Vert \mathbf{x}_j\Vert /\varepsilon} \leq e^{1/\varepsilon}$ in the last equality.
\end{proof}
Using Lemma~\ref{lem:H_tilde error} and applying the union bound on the off-diagonal entries of $\widetilde{H} - H$, we obtain
\begin{equation}
    \Vert \widetilde{H} - H \Vert_F  = \mathcal{O}_p (m^{-1/2}), \label{eq:H_tilde - H frob bound}
\end{equation}

Continuing, we bound the quantities $\max_{i,j}\{u_i u_j\}$ and $\max_{i,j} \{\widetilde{H}_{i,j}\}$ from~\eqref{eq:W_tilde - W frob bound}. Towards that end, we have the following result.
\begin{prop} \label{prop:H_ij and u_i bounds}
Under the conditions of Theorem~\ref{thm:Main theorem}, $\{H_{i,j}\}_{i\neq j}$ and $\{u_i\}_{i=1}^n$ are upper- and lower-bounded by positive constants independent of $m$.
\end{prop}
\begin{proof}
Observe that for $i\neq j$ and for all $m$,
\begin{equation}
    0 < e^{-1/\varepsilon} \leq H_{i,j} = e^{\langle \mathbf{x}_i,\mathbf{x}_j \rangle/\varepsilon} \leq e^{1/\varepsilon}. \label{eq:H_ij upper and lower bounds}
\end{equation}
Since the set of $n\times n$ matrices satisfying the above is compact, and using the fact that $\{u_i\}>0$ can be uniquely determined by $H$ (from Lemma~\ref{lem:existence and uniquness} applied to $H$), there must exist constants $c_u$, $C_u$ independent of $m$ such that
\begin{equation}
    0 < c_u \leq u_i \leq C_u, \label{eq:u_i upper and lower bounds}
\end{equation}
for all $i$ and $m$.
\end{proof}
Consequently, Proposition~\ref{prop:H_ij and u_i bounds} together with Lemma~\ref{lem:H_tilde error} guarantee that 
\begin{equation}
\begin{aligned}
    \max_{i,j}\{u_i u_j\} = \mathcal{O}_p(1), \qquad\qquad \max_{i,j}\{ \widetilde{H}_{i,j} \} = \max_{i,j}\{ {H}_{i,j} \}+ \mathcal{O}_p(m^{-1/2}) = \mathcal{O}_p(1), \label{eq:max_u_i_u_j and max_H_ij_tilde bounds}
    \end{aligned}
\end{equation}

Next, we turn to bound the quantity $\Vert \widetilde{\mathbf{u}} \widetilde{\mathbf{u}}^T - \mathbf{u}\mathbf{u}^T \Vert_F$ from~\eqref{eq:W_tilde - W frob bound}. From  Lemma~\ref{lem:existence and uniquness} applied to $H$ and $\widetilde{H}$, it follows that $\mathbf{u}$ and $\widetilde{\mathbf{u}}$ are unique. Additionally, by Lemma~\ref{lem:H_tilde error} it is clear that $\widetilde{H} \overset{p}{\longrightarrow} H$. Therefore, we also have that $\widetilde{\mathbf{u}} \overset{p}{\longrightarrow} \mathbf{u}$ (as otherwise we have a contradiction to the uniqueness of $\mathbf{u}$ and $\widetilde{\mathbf{u}}$).
Since $\widetilde{W}^{(d)}$ is doubly-stochastic, we have
\begin{equation}
    \sum_{j=1}^n \widetilde{W}^{(d)}_{i,j} = \sum_{j=1}^n \widetilde{u}_i \widetilde{H}_{i,j} \widetilde{u}_{j} = 1. \label{eq:W_tilde doubly-stochstic}
\end{equation}
Let us define the multivariate functions $\{f_i(A,\textbf{v})\}_{i=1}^n$, where $A\in\mathbb{R}^{n\times n}$, $\mathbf{v}\in\mathbb{R}^n$, as
\begin{equation}
    f_i(A,\mathbf{v}) = {\sum_{j=1}^m  v_i A_{i,j} v_j}. \label{eq:f_i def}
\end{equation}
To bound the error $\Vert \mathbf{\widetilde{u}} - \mathbf{u} \Vert$, we expand $f_i(A,\mathbf{v})$ around $(H,\mathbf{u})$ using a first-order Taylor expansion.
Towards that end, Proposition~\ref{prop:H_ij and u_i bounds} can be used to verify that the second-order partial derivatives of $f_i$ in the vicinity of $(H,\mathbf{u})$ are bounded by constants independent of $m$. In particular,
\begin{equation}
    \max_{(A,\mathbf{v})\in \mathcal{B}_1(H,\mathbf{u}) }\left\vert\frac{\partial^2 f_i}{\partial v_k \partial v_j}\right\vert = \mathcal{O}_p(1), \qquad
    \max_{(A,\mathbf{v})\in \mathcal{B}_1(H,\mathbf{u}) }\left\vert\frac{\partial^2 f_i}{\partial v_k A_{m,j}} \right\vert = \mathcal{O}_p(1), \qquad 
    \frac{\partial^2 f_i}{\partial A_{k,j} A_{m,\ell}} = 0,  \label{eq:second derivatives}
\end{equation}
for all $i,j,k,m,\ell$, where $\mathcal{B}_1(H,\mathbf{u})$ is a ball of radius $1$ in Euclidean space around $(H,\mathbf{u})$:
\begin{equation}
    \mathcal{B}_1(H,\mathbf{u}) = \left\{(A,\mathbf{v}): \Vert A - H \Vert_F^2 + \Vert \mathbf{v} - \mathbf{u}\Vert^2 \leq 1 \right\}.
\end{equation}
The choice of the radius of the ball $\mathcal{B}_1(H,\mathbf{u})$ is arbitrary, and is only required to guarantee that the point $(\widetilde{H},\widetilde{\mathbf{u}})$ is included in $\mathcal{B}_1(H,\mathbf{u})$ for sufficiently large $m$.
Therefore, by~\eqref{eq:W_tilde doubly-stochstic} and~\eqref{eq:second derivatives}, the first-order Taylor expansion of $f_i(A,\mathbf{v})$ around $(H,\mathbf{u})$ gives
\begin{align}
     1 = f_i(\widetilde{H},\widetilde{\mathbf{u}}) &= f_i(H,\mathbf{u}) + \sum_{j=1}^n \frac{\partial f_i}{\partial {v}_j}\bigg\vert_{(H,\mathbf{u})} (\widetilde{u}_j - u_j) + \sum_{k,j=1}^n \frac{\partial f_i}{\partial {A}_{k,j}}\bigg\vert_{(H,\mathbf{u})} (\widetilde{H}_{k,j} - H_{k,j}) \nonumber \\
    &+ \mathcal{O}_p(\Vert \widetilde{\mathbf{u}} - \mathbf{u} \Vert^2 ) + \mathcal{O}_p(\Vert \widetilde{{H}} - {H} \Vert_F^2 ). \label{eq:f_i Taylor}
\end{align}
where
\begin{align}
    \frac{\partial f_i}{\partial {v}_j}\bigg\vert_{(H,\mathbf{u})} = \begin{dcases}
    \frac{1}{u_i}, & j=i, \\
    u_i H_{i,j}, & j\neq i,
    \end{dcases}
    \qquad\qquad
    \frac{\partial f_i}{\partial {A}_{k,j}}\bigg\vert_{(H,\mathbf{u})} = 
    \begin{dcases}
    {u_i u_j}, & k=i, \\
    0, & k\neq i,
    \end{dcases}
\end{align}
and we used the fact that $\sum_{j=1}^n u_i H_{i,j} u_j = 1$ ($W^{(d)}$ is doubly-stochastic).
Next, using that $f_i(H,\mathbf{u}) = 1$, denoting $\widetilde{u}_j - u_j := e_j$, and multiplying both hand sides of~\eqref{eq:f_i Taylor} by $u_i$ ($u_i$ is bounded according to Proposition~\ref{prop:H_ij and u_i bounds}), we can write
\begin{equation}
    e_i = -\sum_{j\neq i}^n u_i^2 H_{i,j} e_j - \sum_{j=1}^n u_i^2 u_j (\widetilde{H}_{i,j} - H_{i,j}) +   \mathcal{O}_p(\Vert \mathbf{e}\Vert^2 ) + \mathcal{O}_p(\Vert \widetilde{{H}} - {H} \Vert_F^2 ), \label{eq:e_i taylor expression}
\end{equation}
where $\mathbf{e} = [e_1,\ldots,e_n]^T$.
Consequently, since $H_{i,i} = 0$, writing~\eqref{eq:e_i taylor expression} in matrix form gives
\begin{equation}
    (I_{n} + \left[\operatorname{diag}(\mathbf{u})\right]^2 H)\mathbf{e} = -\left[\operatorname{diag}(\mathbf{u})\right]^2 (\widetilde{H} - H)\mathbf{u} + \mathcal{O}_p(\Vert \mathbf{e} \Vert^2 ) + \mathcal{O}_p(\Vert \widetilde{{H}} - {H} \Vert_F^2 ), \label{eq:e Taylor matrix form}
\end{equation}
where $I_n$ is the $n\times n$ identity matrix. In order to bound the vector $\mathbf{e}$, we must be able to invert the matrix
\begin{equation}
    G:=I_{n} + \left[\operatorname{diag}(\mathbf{u})\right]^2 H, \label{eq:G def}
\end{equation}
which is the subject of the following Lemma. 
\begin{lem} \label{lem:G is invertible}
Under the conditions of Theorem~\ref{thm:Main theorem}, the matrix $G$ from~\eqref{eq:G def} is invertible for all $m$, and $\Vert G^{-1} \Vert_2 \leq C_G$ for some constant $C_G$ independent of $m$.
\end{lem}
\begin{proof}
Notice that $G$ is similar to the matrix
\begin{equation}
    \left[\operatorname{diag}(\mathbf{u})\right]^{-1} G \operatorname{diag}(\mathbf{u}) = I_n + \operatorname{diag}(\mathbf{u}) H \operatorname{diag}(\mathbf{u}) = I_n + W^{(d)}.
    \end{equation}
Therefore, $G$ is invertible if $I_n + W^{(d)}$ is invertible.
Since $W^{(d)}$ is symmetric and doubly-stochastic, its largest eigenvalue is exactly $1$, and $\lambda_{\min}\{W^{(d)}\} \geq -1$. Moreover, since $W_{i,j}>0$ for all $i\neq j$, we have that $\{(W^{(d)})^2\}_{i,j} > 0$ for all $i,j$. Therefore, by Lemma 8.4.3 in~\cite{horn_johnson_2012} $W^{(d)}$ has only one eigenvalue with maximal absolute-value (which is $1$). Hence, $\lambda_{\min}\{W^{(d)}\} > -1$, and we obtain that 
\begin{equation}
    \lambda_{\min} \{ G \} = \lambda_{\min} \{ I_n + W^{(d)}\} = 1 + \lambda_{\min} \{ W^{(d)} \} > 0.
\end{equation}
The fact that $\Vert G ^{-1} \Vert_2$ is bounded by some constant independent of $m$ is established by Proposition~\ref{prop:H_ij and u_i bounds} (since the set of all possible matrices $G$ that satisfy~\eqref{eq:H_ij upper and lower bounds} and~\eqref{eq:u_i upper and lower bounds} is compact). 
\end{proof}
Using~\eqref{eq:e Taylor matrix form} together with Lemma~\ref{lem:G is invertible} and Proposition~\ref{prop:H_ij and u_i bounds}, we have that
\begin{align}
    \Vert \mathbf{e} \Vert &\leq C_G \Vert \left[\operatorname{diag}(\mathbf{u})\right]^2 (\widetilde{H} - H)\mathbf{u} \Vert + \mathcal{O}_p(\Vert \mathbf{e} \Vert^2 ) + \mathcal{O}_p(\Vert \widetilde{{H}} - {H} \Vert_F^2 ) \nonumber \\
    &\leq C_G \Vert \operatorname{diag}(\mathbf{u})\Vert_2^2 \cdot \Vert \widetilde{H} - H \Vert_2 \cdot \Vert \mathbf{u}\Vert_2 +  \mathcal{O}_p(\Vert \mathbf{e} \Vert^2 ) + \mathcal{O}_p(\Vert \widetilde{{H}} - {H} \Vert_F^2 ) \nonumber \\
    &= \mathcal{O}_p(\Vert \widetilde{{H}} - {H} \Vert_F ) +  \mathcal{O}_p(\Vert \mathbf{e} \Vert^2 ) + \mathcal{O}_p(\Vert \widetilde{{H}} - {H} \Vert_F^2 ), \label{eq:e bound derivation}
\end{align}
where we used the inequality $\Vert \widetilde{H} - H \Vert_2\leq \Vert \widetilde{H} - H \Vert_F$.
From~\eqref{eq:H_tilde - H frob bound},~\eqref{eq:e bound derivation}, and the fact that $\widetilde{H} \overset{p}{\longrightarrow} H$, $\widetilde{\mathbf{u}} \overset{p}{\longrightarrow} \mathbf{u}$, it follows that
\begin{equation}
    \Vert \mathbf{e} \Vert = \Vert \widetilde{\mathbf{u}} - \mathbf{u} \Vert = \mathcal{O}_p (m^{-1/2}). \label{eq:e bound}
\end{equation}
Consequently,
\begin{equation}
    \Vert \widetilde{\mathbf{u}} \widetilde{\mathbf{u}}^T - \mathbf{u}\mathbf{u}^T \Vert_F = \Vert (\mathbf{u} + \mathbf{e})(\mathbf{u} + \mathbf{e})^T - \mathbf{u}\mathbf{u}^T \Vert_F = \Vert \mathbf{e}\mathbf{u}^T + \mathbf{u}\mathbf{e}^T +\mathbf{e}\mathbf{e}^T \Vert_F \leq 2 \Vert \mathbf{u} \Vert \cdot \Vert \mathbf{e} \Vert + \Vert \mathbf{e} \Vert^2 = \mathcal{O}_p (m^{-1/2}), \label{eq:u_tilde u_tilde_T - u u_T fro bound}
\end{equation}
where we used Proposition~\ref{prop:H_ij and u_i bounds} to bound $\Vert \mathbf{u} \Vert$. Overall, substituting~\eqref{eq:u_tilde u_tilde_T - u u_T fro bound},~\eqref{eq:H_tilde - H frob bound}, and~\eqref{eq:max_u_i_u_j and max_H_ij_tilde bounds} into~\eqref{eq:W_tilde - W frob bound}, we arrive at the required result
\begin{equation}
    \Vert \widetilde{W}^{(d)} - W^{(d)} \Vert_F = \mathcal{O}_p (m^{-1/2}).
\end{equation}

\end{appendices}

\bibliographystyle{plain}
\bibliography{mybib}

\end{document}